\documentclass{article}
\usepackage{arxiv}

\usepackage{natbib}
\usepackage{subcaption}

\usepackage{amsmath, amsfonts}
\usepackage{amssymb}
\usepackage{mathtools}
\usepackage{amsthm}

\usepackage{mathabx}
\usepackage{bbm}
\usepackage{dsfont}
\usepackage{stmaryrd} 

\newcommand{\ie}{\textit{i}.\textit{e}., }
\newcommand{\eg}{\textit{e}.\textit{g}.\ }
\newcommand{\dfour}{\texttt{d4}}
\newcommand{\dfourpm}{\ensuremath{\dfour^{\pm}}}

\usepackage{tikz}
\usetikzlibrary{arrows.meta, positioning, calc}

\usepackage{booktabs}

\usepackage{url}

\theoremstyle{thmstyleone}%
\newtheorem{theorem}{Theorem}
\newtheorem*{claim*}{Claim}%
\newtheorem{corollary}[theorem]{Corollary}%
\newtheorem{lemma}[theorem]{Lemma}%

\theoremstyle{thmstyletwo}%
\newtheorem{example}{Example}%
\newtheorem{remark}{Remark}%

\theoremstyle{thmstylethree}%
\newtheorem{definition}{Definition}%

\usepackage[lighttt]{lmodern}
\usepackage{listings}
\usepackage{authblk} 

\title{Solving Hard XAI Queries Based on a Compiled Dual-Rail Encoding}

\author[1,2]{\thanks{corresponding author: arthur.ledaguenel@gmail.com} Arthur Ledaguenel}
\author[1]{Florent Capelli}
\author[1]{Jean-Marie Lagniez}

\affil[1]{CRIL, Université d'Artoi}
\affil[2]{LISIC, ULCO}

\begin{document}
\maketitle

\begin{abstract}
The widespread adoption of artificial intelligence (AI) within real-world applications has raised a lot of concerns regarding their trustworthiness, especially in critical applications. The field of eXplainable AI (XAI) has emerged with the objective of providing explanations to the users about the decisions made by AI systems. Several explanations for boolean classifiers have been introduced in the literature, including abductive and contrastive explanations, each giving a different insight on the decision of the classifier. However, computing an explanation for a decision of a boolean classifier is a hard problem in general. One way to deal with this complexity is to rely on a compiled representation of the classifier for which each explanation can be computed efficiently. Unfortunately, we prove in this paper that several classes of abductive explanations, remain hard to compute even for Ordered Binary Decision Diagrams, one of the most tractable subsets of the knowledge compilation map. Included in such classes are shorter abductive explanations or abductive explanations that include the explainee's preferences. To recover the benefits of working with compiled representations, we show that a proper representation of the dual-rail encoding of the classifier can be used to compute efficiently these classes of explanations.
\end{abstract}


\section{Introduction}
The widespread adoption of artificial intelligence (AI) within real-world applications, powered by the recent advancements of machine learning, has raised a lot of concerns regarding their trustworthiness, especially in critical applications. Building trust in AI systems depends on characteristics like fairness, robustness, and explainability, each of which has been the subject of active research in the past decade. Hence, the field of eXplainable AI (XAI) has emerged with the objective of providing explanations to the users about the decisions made by AI systems. Within this field, formal explainability \citep{ignatiev_abduction-based_2019} designs explanations with provable guarantees and studies the complexity of computing these explanations based on the type of model used in the AI system. On boolean classifiers specifically, several explanations have been introduced in the literature, each giving a different insight on the decision of the classifier. For example, abductive explanations highlight a set of features in the input instance whose values are sufficient to fix the decision of the classifier \citep{ignatiev_abduction-based_2019}, whereas contrastive explanations suggest a minimal set of features that can be changed in the input instance to swap the decision of the classifier \citep{audemard_computation_2024}. Other types of explanations such as the Shapley value have also been adapted to boolean classifiers \citep{letoffe_towards_2025}.

Computing an explanation for a decision of a boolean classifier is a hard problem in general. In fact, one of the main focus of the literature in formal explainability has been to characterize for which types of classifiers a given explanation can be computed efficiently (typically in time polynomial in the size of the classifier). This \textit{cartographic} approach of formal explainability draws a natural parallel with knowledge compilation \citep{darwiche_knowledge_2002}, a field whose goal is to characterize propositional languages representing boolean functions according to their succinctness and the set of tractable queries and transformations they support. Hence, the complexity of computing explanations based on a compiled representation of the classifier has been heavily explored recently \citep{audemard_tractable_2020,huang_tractable_2022}. It is known for instance that computing subset-minimal abductive explanations (also called sufficient reasons or prime implicant explanations) can be done efficiently for $\mathtt{d\text{-}DNNF}$ classifiers. Unfortunately, we prove in this paper that several classes of abductive explanations remain hard even for Ordered Binary Decision Diagrams ($\mathtt{OBDD}$), one of the most tractable subsets of the knowledge compilation map. This is for instance the case of shortest (or minimal-size) explanations \citep{audemard_deriving_2024} or optimal abductive explanations with respect to an additive utility function representing the explainee's preferences \citep{audemard_preferred_2022}. Algorithms based on SAT-like solvers have been proposed for these types of explanations, but they lack the guarantees that compiled representations can provide. The aim of this paper is precisely to fill that gap.

At the core of computing abductive explanations lies the complexity of reasoning about the implicants of a formula. The dual-rail encoding \citep{bryant_cosmos_1987} has already been used to generate or enumerate selected implicants of a formula \citep{manquinho_prime_1997,palopoli_algorithms_1999,jabbour_enumerating_2014,previti_prime_nodate}, but we believe that its full potential has not yet been explored. In particular, we show that a proper representation of the dual-rail encoding of a classifier can be used to efficiently compute several hard explanations, such as shortest abductive explanations, preferred abductive explanations or coverage-based explanations in the presence of domain constraints \citep{cooper_abductive_2023}. We also show that the dual-rail of a $\mathtt{CNF}$ can be easily represented as another $\mathtt{CNF}$, which allows us to use existing knowledge compilers \textit{as-is}. Finally, we measure experimentally the complexity gap between compiling the original function and its dual-rail.

\paragraph{Outline.} We first review related works in Section \ref{sec:related}. We define preliminary notions on boolean functions, knowledge compilation and XAI queries in Section \ref{sec:preliminaries}. We define the dual-rail encoding in \ref{sec:dual-rail} and demonstrate its potential for solving hard XAI queries in Section \ref{sec:hardxai}. We discuss how to compile dual-rail encodings into tractable representations in Section \ref{sec:compilation} and test this approach experimentally in Section \ref{sec:xp}. Finally, we conclude our paper and discuss future research directions in Section \ref{sec:conclusion}.

\section{Related works} \label{sec:related}
The complexity of computing abductive explanations based on the representation of the classifier has been studied in various contexts \citep{shih_symbolic_2018,ignatiev_abduction-based_2019,audemard_tractable_2020,audemard_computational_2021,huang_tractable_2022}. Selecting a preferred abductive explanation based on a model of the explainee's preferences was studied in \citep{audemard_preferred_2022,bounia_impact_2023,bounia_enhancing_2025}, with a focus on majority reasons and probabilistic explanations for random forests. The impact of domain constraints on explanations redundancy and computational complexity was highlighted in \citep{cooper_tractability_2021,gorji_sufficient_2022,cooper_abductive_2023,audemard_deriving_2024}.


The dual-rail encoding \citep{bryant_cosmos_1987} has already been used to generate or enumerate selected prime implicants of a formula \citep{palopoli_algorithms_1999,manquinho_prime_1997,jabbour_enumerating_2014,previti_prime_nodate}. However, dual-rail encodings are typically paired with SAT oracles or integer linear programming solvers rather than compiled into a tractable representation, as we explore in this paper. Moreover, we show that the range of queries they support go beyond what they have been used for so far, notably regarding XAI queries.

\section{Preliminaries} \label{sec:preliminaries}
For a positive integer $i$, we denote by $[i]$ the set of integers $\{1, . . . , i\}$. Given any set or vector $\mathbf{S}$ indexed by $[n]$ and a subset $I \subset [n]$ of indexes, we note $\mathbf{S}_I$ the set $\{S_i | i \in I\}$ the set of elements of index in $I$.

\paragraph{Terms and instances.} We assume a set of \textbf{boolean variables} $\mathbf{X}=\{X_1, ..., X_n\}$. A literal $l$ on a variable $X$ is either $x$ or $\overline{x}$. We also note $\overline{l}$ the literal opposite to $l$. A term is a set of consistent literals (it does not contain a literal and its opposite). The set of literals (resp. terms) on $\mathbf{X}$ is noted $\mathtt{lits}(\mathbf{X})$ (resp. $\mathtt{terms}(\mathbf{X})$). For any two terms $\mathbf{t}$ and $\mathbf{h}$, if $\mathbf{h} \subset \mathbf{t}$ we say that $\mathbf{h}$ is a sub-term of $\mathbf{t}$ or equivalently that $\mathbf{t}$ is an extension of $\mathbf{h}$. Two terms are consistent if they agree on mutual variables. An instance $\mathbf{x}$ is a term that contains one literal for each variable. We note $\mathtt{var}(\mathbf{t})$ the set of variables featured in $\mathbf{t}$. The set of all instances on $\mathbf{X}$ is noted $\mathcal{X} = \{0, 1\}^{\mathbf{X}} \simeq \{0, 1\}^n$. A instance $\mathbf{x} \in \mathcal{X}$ is equivalently viewed as a boolean vector or a set of literals such that $x_i=1$ (resp. $x_i=0$) iff $\mathbf{x}$ contains the literal $x_i$ (resp. $\overline{x_i}$). The size (resp. cardinality) of a term is its number of literals (resp. positive literals).

\paragraph{Boolean functions.} A \textbf{boolean function} on $\mathbf{X}$ is a function $f:\mathcal{X} \to \{0,1\}$. An instance $\mathbf{x} \in \mathcal{X}$ is called a model of $f$ iff $f(\mathbf{x})=1$. An \textbf{implicant} of $f$ is a term $\mathbf{t}$ such that any extension of $\mathbf{t}$ is a model of $f$. The set of models (resp. implicants) of $f$ is noted $\mathtt{mods}(f)$ (resp. $\mathtt{ips}(f)$). A boolean function is said consistent if it has at least one model and is said inconsistent otherwise. We also note $\mathbf{t} \models f$ to indicate that a term $\mathbf{t}$ is an implicant of $f$. An implicant is \textbf{prime} if none of its strict sub-terms is also an implicant. Given two boolean functions $f$ and $g$, we say that $f$ implies $g$, noted $f \models g$, iff any model of $f$ is a model of $g$ (\ie $\mathtt{mods}(f) \subset \mathtt{mods}(g)$). We follow standard definitions and notations for the negation, conjunction or disjunction of boolean functions using the symbols $\neg, \land$ and $\lor$. The conditioning of $f$ on a term $\mathbf{t}$ is noted $f|\mathbf{t}$.


\paragraph{Knowledge compilation.} A representation language $\mathtt{L}$ defines a syntax for writing boolean functions as strings of symbols. We provide below rudimentary definitions for boolean circuits and binary decision diagrams.

\begin{definition}[Boolean circuits]
    A \textbf{boolean circuit} $\phi(\mathbf{X})$ in Negation Normal Form ($\mathtt{NNF}$) \citep{darwiche_knowledge_2002} is a labeled directed acyclic graph where each leaf node is labeled with $\top, \bot$, a variable $X$ or its negation $\neg X$ and internal nodes are labeled with logical connectives $\land$ or $\lor$. For implicity we assume w.l.o.g. that each internal node $n$ has exactly two children noted $n_l$ and $n_r$. Its size $|\phi|$ is its number of edges. Semantics of boolean circuits follow the usual semantics of logical connectors $\neg, \land$ and $\lor$. A $\land$-node is \textbf{decomposable} if its children do not share variables. A $\mathtt{NNF}$ circuit is in Decomposable Negation Normal Form ($\mathtt{DNNF}$) iff all of its $\land$-nodes are decomposable. A $\lor$-node $u$ is \textbf{deterministic} if its children are inconsistent. A $\mathtt{DNNF}$ circuit is in deterministic Decomposable Negation Normal Form ($\mathtt{d\text{-}DNNF}$) iff all of its $\lor$-nodes are deterministic. A \textbf{vtree} $T$ over $\mathbf{X}$ is a full, rooted, binary tree whose leaves are in one-to-one correspondence with $\mathbf{X}$. For any node $t \in T$, we node $t_l$ and $t_r$ its two children. A $\mathtt{DNNF}$ $\phi$ respects a vtree $T$ iff for any $\land$-node $n$ in $\phi$ there is a node $t \in T$ such that $\mathtt{var}(n_l) \subset \mathtt{var}(t_l)$ and $\mathtt{var}(n_r) \subset \mathtt{var}(t_r)$. The class of structured $\mathtt{DNNF}$ (resp. $\mathtt{d\text{-}DNNF}$) is noted $\mathtt{SDNNF}$ (resp. $\mathtt{d\text{-}SDNNF}$).
\end{definition}

\begin{definition}[Ordered Binary Decision Diagrams]
    A \textbf{binary decision diagram} ($\mathtt{BDD}$) $\phi(\mathbf{X})$ is a labeled directed acyclic graph with two leaf nodes labeled $\top$ and $\bot$. Each internal node is labeled with a variable in $\mathbf{X}$ and has two outgoing edges labeled $0$ and $1$. A path $p$ in $\phi$ is an accepting path for an instance $\mathbf{x} \in \mathcal{X}$ if its last node is labeled $\top$ and for every internal node $n$ labeled with variable $X_i$ in $p$, the outgoing edge labeled $1$ (resp. $0$) is in $p$ if $x_i=1$ (resp. if $x_i=0$). A model of $\phi$ is an instance with at least one accepting path. A \textbf{free} binary decision diagram ($\mathtt{FBDD}$) is a $\mathtt{BDD}$ where any variable appears at most once in any path from the root to a leaf. An binary decision diagram respects an order $\sigma:\mathbf{X} \to [n]$ on variables iff for any two variables $X_i$ and $X_j$ such that $X_i$ precedes $X_j$ in a path of $\phi$, then $\sigma(X_i) <\sigma(X_j)$. An \textbf{ordered} binary decision diagram ($\mathtt{OBDD}$) is a $\mathtt{BDD}$ that respects any order. Notice that an $\mathtt{OBDD}$ is necessarily a $\mathtt{FBDD}$.
\end{definition}


Throughout the paper, a formula $\phi \in \mathtt{L}$ will be conflated with the boolean function it represents. Its size is noted $|\phi|$ and any notation introduced for boolean function is extended to formulas. We assume the reader is familiar with the standard subsets of Negation Normal Form (NNF) formulas listed in Table \ref{tab:languages} and refer to \citep{darwiche_knowledge_2002,amarilli_connecting_2020} or supplementary materials for more details. Knowledge compilation \citep{darwiche_knowledge_2002} provides a strategy to deal with the intractability of propositional reasoning: a formula in a general representation language (\eg $\mathtt{CNF}$) is compiled into a representation in a tractable language (\eg $\mathtt{OBDD}$) for which desired queries can be answered in polynomial time. Therefore, given any concrete representation language $\mathtt{L}$ for boolean functions, one is interested in the list of queries or transformations that can be performed on a formula $\phi \in \mathtt{L}$ in time polynomial in $|\phi|$. Standard queries include deciding if the formula is consistent or valid (all instances are models), deciding if the formula implies a given clause or is implied by a given term, counting or enumerating the models of the formula. Standard transformations include conditioning the formula on a set of literals, computing the conjunction or disjunction of two formulas in $\mathtt{L}$.

\begin{table}[h!]
\centering
\begin{tabular}{l  l}
    \toprule
    \textbf{Acronym} & \textbf{Name}\\
    \midrule
    $\mathtt{CNF}$ & Conjunctive Normal Form \\
    $\mathtt{d\text{-}DNNF}$ & Deterministic Decomposable NNF \\
    $\mathtt{dec\text{-}DNNF}$ & Decision Decomposable NNF \\
    $\mathtt{OBDD}$ & Ordered Binary Decision Diagram \\
    \bottomrule
\end{tabular}
\caption{Standard representation languages.}
\label{tab:languages}
\end{table}

\paragraph{XAI queries.} All XAI queries considered in this paper are build around the concept of an abductive explanation.
\begin{definition}
    Let $f \in \mathcal{F}_n$ and $\mathbf{x} \in \mathcal{X}_n$ such that $f(\mathbf{x})=1$ (resp. $f(\mathbf{x})=0$), an \textbf{abductive explanation} for the decision $(\mathbf{x}, f)$ is an implicant of $f$ (resp. $\neg f$) that is consistent with $\mathbf{x}$. The set of abductive explanations for a model $\mathbf{x}$ of $f$ is noted $\mathtt{axp}(f, \mathbf{x})$.
\end{definition}
To avoid the redundancy of abductive explanations, one can select minimal abductive explanations. The minimality criterion can either be a general minimality principle (such as size or subset minimality) or can be defined to model the preferences the explainee. A \textbf{sufficient reason} (or prime-implicant explanation) is a subset-minimal abductive explanation. A \textbf{shortest reason} is an abductive explanation that has minimal size. Notice that a shortest reason is necessarily a sufficient reason.

Preferences of the explainee can be represented both quantitatively, as an additive dis-utility function over variables, or qualitatively, as a stratification of variables from the least useful to the most useful.
\begin{definition}
    Let $\mathbf{w} \in (\mathbb{R}^+)^n$ be a vector of weights on variables, the dis-utility of an explanation $\mathbf{t}$ is the sum of the weights over variables featured in the explanation, \ie $w(\mathbf{t}) = \sum_{X_i \in \mathtt{var}(\mathbf{t})} w_i$. An \textbf{optimal explanation} with respect to $\mathbf{w}$ is one that minimizes its dis-utility.
\end{definition}

\begin{definition}    
    Let $\mathbf{P}=(\mathbf{P}_i)_{i \in [p]}$ form a stratified partition of $\mathbf{X}$, an explanation $\mathbf{t}_1$ is shorter to an explanation $\mathbf{t}_2$ with respect to $\mathbf{P}$, noted $\mathbf{t}_1 \prec_{\mathbf{P}} \mathbf{t}_2$, iff its stratified size is shorter following a lexicographic order, \ie:
    \begin{equation}
        \begin{split}
            \exists i \in [p], \quad & |\mathtt{var}(\mathbf{t}_1) \cap \mathbf{P}_i| < |\mathtt{var}(\mathbf{t}_2) \cap \mathbf{P}_i| \\
            \forall j < i, \quad & |\mathtt{var}(\mathbf{t}_1) \cap \mathbf{P}_j| = |\mathtt{var}(\mathbf{t}_2) \cap \mathbf{P}_j|
        \end{split}
    \end{equation}
    A \textbf{shortest explanation} with respect to a \textbf{stratification} $\mathbf{P}=(\mathbf{P}_i)_{i \in [p]}$ is one that is minimal for $\prec_{\mathbf{P}}$.
\end{definition}

\begin{remark}
    A subset-minimal explanation based on stratified preferences is defined in \citep{audemard_preferred_2022}.
\end{remark}

\begin{example} \label{ex:1}
    Let $\phi = (X_1 \lor \neg X_2) \land (X_1 \lor X_3)$ be a propositional formula in $\mathtt{CNF}$. Instance $\mathbf{x}=x_1 \overline{x_2} x_3$ is one of the $5$ models of $\phi$. The set of abductive explanations for $(\mathbf{x}, \phi)$ is $\mathtt{axp}(\phi, \mathbf{x})=\{x_1, \overline{x_2} x_3, x_1 \overline{x_2} x_3\}$, among which two are sufficient reasons (\ie $\mathtt{srs}(\phi, \mathbf{x})=\{x_1, \overline{x_2} x_3\}$) and only one is a shortest reason (\ie $\{x_1\}$). The optimal explanation with respect to dis-utility scores $\mathbf{w}=(3, 2, 2)$ is $\{x_1\}$ while the shorter explanation with respect to stratified preferences $\mathbf{P}=(\{x_1\}, \{x_2 x_3\})$ is $\{\overline{x_2} x_3\}$.
\end{example}

\begin{table*}[h!]
\centering
\begin{tabular}{l  l}
    \toprule
    \textbf{Acronym} & \textbf{Description}\\
    \midrule
    \multicolumn{1}{l}{\textbf{Optimization}} & \\
    $\mathtt{MWM}$ & Finds a minimum weight model \\
    $\mathtt{MPE}$ & Finds a most probable model \\
    $\mathtt{SSM}$ & Finds a shortest model with respect to stratified preferences \\
    $\mathtt{ST}$ & Finds a shortest term in any predicate $\mathtt{pred}$ (trivial for $\mathtt{mods}$) \\
    \midrule
    \multicolumn{1}{l}{\textbf{Counting}} & \\
    $\mathtt{CT}$ & Counts the models of $\phi$ \\
    $\mathtt{CTC}$ & Counts the models of $\phi$ of cardinality $k$ for any $0 \leq k \leq n$ \\
    \midrule
    \multicolumn{1}{l}{\textbf{Enumeration}} & \\
    $\mathtt{Enum}$ & Enumerate the models of $\phi$ \\
    \midrule
    \multicolumn{1}{l}{\textbf{Transformations}} & \\
    $\mathtt{CD}$ & Conditions $\phi$ on a given term \\
    $\mathtt{BC}\land$ & Computes the conjunction of two formulas $\phi$ and $\varphi$ \\
    
    \bottomrule
\end{tabular}
\caption{Acronym of queries and their descriptions.}
\label{tab:queries}
\end{table*}

\paragraph{Algebraic model counting.} Algebraic model counting is an abstract framework introduced in \cite{kimmig_algebraic_2017} that generalizes many counting and optimization queries on boolean functions based on a sum of products computation over models with suitable operators from a semiring structure.

\begin{definition}
    A (commutative) \textbf{semiring} is a structure $(\mathcal{A}, \oplus, \otimes, e^{\oplus}, e^{\otimes})$ where:
    \begin{itemize}
        \item addition $\oplus$ is an associative and commutative binary operation over $\mathcal{A}$
        \item multiplication $\otimes$ is an associative binary operation over $\mathcal{A}$
        \item $\otimes$ distributes over $\oplus$
        \item $e^{\oplus}$ is the neutral element of $\oplus$
        \item $e^{\otimes}$ is the neutral element of $\otimes$
        \item $e^{\oplus}$ is an annihilator for $\otimes$
    \end{itemize}
    In a commutative semiring $\otimes$ is also commutative. An optimization semiring is a semiring where $\oplus$ is either $\max$ or $\min$.
\end{definition}

Given a formula $\phi$, a commutative semiring $\mathcal{SR}=(\mathcal{A}, \oplus, \otimes, e^{\oplus}, e^{\otimes})$ and an attribution function $\alpha: \mathtt{lits}(\mathbf{X}) \to \mathcal{A}$, the solution to the algebraic model counting query $(\phi, \mathcal{SR}, \alpha)$ is:
\begin{equation} \label{eq:amc}
    \alpha(\phi) = \bigoplus_{\mathbf{x} \in \mathtt{mods}(\phi)} \bigotimes_{l \in \mathbf{x}} \alpha(l)
\end{equation}

Fixing a specific semiring yields a specific query, which can be easier than general algebraic model counting. As we prove below, all counting and optimization problems described in Table \ref{tab:queries} can be derived from algebraic model counting by choosing a specific semiring.

\begin{itemize}
    \item \textbf{Model counting} ($\mathtt{CT}$): counts the models of $\phi$.

        \begin{itemize}
            \item Semiring $(\mathbb{N}, +, \times, 0, 1)$
            \item Attribution function: $\forall l\in \mathtt{lits}(\mathbf{X)}, \alpha(l)=1$
        \end{itemize}
    
        The associativity and commutativity of $+$ and $\times$ are trivial, as for the distributivity of $\times$ over $+$.
        
        \begin{claim*}
            $$\alpha(\phi) = \sum_{\mathbf{x} \in \mathtt{mods}(\phi)} \prod_{l \in \mathbf{x}} \alpha(l) = |\mathtt{mods}(\phi)| $$
        \end{claim*}
        
        \begin{proof}
            For any instance $\mathbf{x} \in \mathcal{X}$, we have:
                \begin{equation}
                    \prod_{l \in \mathbf{x}} \alpha(l) = \prod_{l \in \mathbf{x}} 1 = 1
                \end{equation}
                Therefore:
                \begin{equation}
                    \begin{split}
                        \alpha(\phi) & = \sum_{\mathbf{x} \in \mathtt{mods}(\phi)} \prod_{l \in \mathbf{x}} \alpha(l) \\
                        & = \sum_{\mathbf{x} \in \mathtt{mods}(\phi)} 1 \\
                        & = |\mathtt{mods}(\phi)|
                    \end{split}
                \end{equation}
        \end{proof}

    \item \textbf{Model counting by cardinality} ($\mathtt{CTC}$): counts the models of $\phi$ of cardinality $k$ for any $0 \leq k \leq n$
        \begin{itemize}
            \item Semiring $(\mathbb{N}^{n+1}, +^{n+1}, \bowtie^{n+1}, \mathbf{0}^{n+1}, \delta^0)$ with:
                \begin{equation*}
                    \begin{split}
                        \forall 0 \leq k \leq n, \quad & +^{n+1}(\mathbf{a}, \mathbf{b})_k = a_k + b_k \\
                        & \bowtie^{n+1}(\mathbf{a}, \mathbf{b})_k = \sum_{0 \leq i \leq k} a_i \times b_{k-i} \\
                        \forall 0 \leq i \leq n, \quad & \delta^k_i = 1 \mbox{ if } k=i \mbox{ and } 0 \mbox{ otherwise}
                    \end{split}
                \end{equation*}
            \item Attribution function:
            \begin{equation}
                \begin{split}
                    \forall i \in [n], \quad & \alpha(x_i)= \delta^1 \\
                    & \alpha(\overline{x_i})= \delta^0
                \end{split}
            \end{equation}
        \end{itemize}
    
        The associativity and commutativity of $+^{n+1}$ are trivial, as for the commutativity of $\bowtie^{n+1}$.
        
        \begin{claim*}
            $\bowtie^{n+1}$ is associative.
        \end{claim*}
        
        \begin{proof}
        For any $\mathbf{a}, \mathbf{b} \in \mathbb{N}^{n+1}$ and $0 \leq k \leq n$:
            \begin{equation*}
                \begin{split}
                    \bowtie^{n+1}(\bowtie^{n+1}(\mathbf{a}, \mathbf{b}), \mathbf{c})_k & = \sum_{0 \leq i \leq k} \bowtie^{n+1}(\mathbf{a}, \mathbf{b})_i \times c_{k-i} \\
                    & = \sum_{0 \leq i \leq k} (\sum_{0 \leq j \leq i} a_j \times b_{i-j}) \times c_{k-i} \\
                    & = \sum_{0 \leq j \leq k} a_j \times (\sum_{j \leq i \leq k} b_{i-j} \times c_{k+j-i}) \\
                    & = \sum_{0 \leq j \leq k} a_j \times (\sum_{0 \leq i \leq k-j} b_{i} \times c_{k-i}) \\
                    & = \sum_{0 \leq j \leq k} a_j \times \bowtie^{n+1}(\mathbf{b}, \mathbf{c})_{k-j} \\
                    & = \bowtie^{n+1}(\mathbf{a}, \bowtie^{n+1}(\mathbf{b}, \mathbf{c}))_k
                \end{split}
            \end{equation*}
            Therefore:
            $$ \bowtie^{n+1}(\bowtie^{n+1}(\mathbf{a}, \mathbf{b}), \mathbf{c}) = \bowtie^{n+1}(\mathbf{a}, \bowtie^{n+1}(\mathbf{b}, \mathbf{c}))$$
        \end{proof}
        
        \begin{claim*}
            $\bowtie^{n+1}$ distributes over $+^{n+1}$.
        \end{claim*}
        
        \begin{proof}
        For any $\mathbf{a}, \mathbf{b}, \mathbf{c} \in \mathbb{N}^{n+1}$ and $0 \leq k \leq n$:
        \begin{equation*}
        \begin{split}
            \bowtie^{n+1}(\mathbf{a}, +^{n+1}(\mathbf{b}, \mathbf{c}))_k & = \sum_{0 \leq i \leq k} a_i \times +^{n+1}(\mathbf{b}, \mathbf{c})_{k-i} \\
            & = \sum_{0 \leq i \leq k} a_i \times (b_{k-i} + c_{k-i}) \\
            & = \sum_{0 \leq i \leq k} a_i \times b_{k-i} + \sum_{0 \leq i \leq k} a_i \times c_{k-i} \\
            & = \bowtie^{n+1}(\mathbf{a}, \mathbf{b})_k + \bowtie^{n+1}(\mathbf{a}, \mathbf{c})_k \\
            & = +^{n+1}(\bowtie^{n+1}(\mathbf{a}, \mathbf{b}), \bowtie^{n+1}(\mathbf{a}, \mathbf{c}))_k
        \end{split}
        \end{equation*}
        Therefore:
        $$ \bowtie^{n+1}(\mathbf{a}, +^{n+1}(\mathbf{b}, \mathbf{c})) = +^{n+1}(\bowtie^{n+1}(\mathbf{a}, \mathbf{b}), \bowtie^{n+1}(\mathbf{a}, \mathbf{c})) $$
    
        In other words: $\bowtie^{n+1}$ distributes over $+^{n+1}$.
    \end{proof}
    
    \begin{claim*}
        For any $0 \leq k \leq n$,
        $$\alpha(\phi)_k = |\mathtt{mods}_k(\phi)|$$
        where $\mathtt{mods}_k(\phi)$ is the set of models of $\phi$ of cardinality $k$.
    \end{claim*}
    
    \begin{proof}
        First, let us show that for any $0 \leq i,j \leq n$, we have $\bowtie^{n+1}(\delta^i, \delta^j) = \delta^{i+j}$.
    
        For any $0 \leq k \leq n$
            \begin{equation}
                \bowtie^{n+1}(\delta^i, \delta^j)_k = \sum_{0 \leq m \leq k} \delta^i_m \times \delta^j_{k-m}
            \end{equation}
        Notice that $\delta^i_m \times \delta^j_{k-m}$ is equal to $1$ iff $m=i$ and $k-m=j$. This happens only once if $k=i+j$ in $m=i$ and cannot happen if $k \ne i+j$. Thus $\bowtie^{n+1}(\delta^i, \delta^j) = \delta^{i+j}$.
            
        In particular, for any instance $\mathbf{x} \in \mathcal{X}$, every positive literal in $\mathbf{x}$ mapped to $\delta^1$ increases the $\bowtie^{n+1}$-product by $1$ and negative literals mapped to $\delta^0$ do not impact the $\bowtie^{n+1}$-product. Hence, $\bowtie^{n+1}_{l \in \mathbf{x}} \alpha(l) = \delta^k$ where $k$ is the cardinality of $\mathbf{x}$.
            
        Therefore, for any $0 \leq k \leq n$:
            \begin{equation}
                \begin{split}
                    \alpha(\phi)_k & = (+^{n+1}_{\mathbf{x} \in \mathtt{mods}(\phi)} \bowtie^{n+1}_{l \in \mathbf{x}} \alpha(l))_k \\
                    & = (+^{n+1}_{\mathbf{x} \in \mathtt{mods}(\phi)} \delta^{|\mathbf{x}|})_k \\
                    & = \sum_{\mathbf{x} \in \mathtt{mods}(\phi)} \delta^{|\mathbf{x}|}_k = \sum_{\mathbf{x} \in \mathtt{mods}(\phi)} \mathds{1}(|\mathbf{x}|=k) \\
                    & = |\mathtt{mods}_k(\phi)|
                \end{split}
            \end{equation}
            where $\mathds{1}$ is the indicator function that returns $1$ when the input is true and $0$ otherwise.
    \end{proof}
    
    \item \textbf{Minimum Weight Model} ($\mathtt{MWM}$): finds a minimum weight model.
        \begin{itemize}
            \item Semiring $(\mathbb{R} \cup \{\infty\}, \min, +, \infty, 0)$ where $\min(\infty, z) = z$ and $+(\infty, z) = \infty$ for any $z \in \mathbb{R}$
            \item Attribution function: 
            $$\forall i \in [n], \alpha(x_i)= 1-\alpha(\overline{x_i}) \in [0,1]$$
        \end{itemize}
    
        The associativity and commutativity of $\min$ and $+$ are trivial, as for the distributivity of $+$ over $\min$. It comes directly from the definition of the semiring that:
    $$ \alpha(\phi) = \min_{\mathbf{x} \in \mathtt{mods}(\phi)} ( \sum_{l \in \mathbf{x}} \alpha(l))$$
    
    \item \textbf{Most Probable Explanation} ($\mathtt{MPE}$): finds a most probable model.
        \begin{itemize}
            \item Semiring $([0,1], \max, \times, 0, 1)$
            \item Attribution function: 
            $$\forall i \in [n], \alpha(x_i)= 1-\alpha(\overline{x_i}) \in [0,1]$$
        \end{itemize}
    
    The associativity and commutativity of $\max$ and $\times$ are trivial, as for the distributivity of $\times$ over $\max$. If we note $p_i=\alpha(x_i)$ for all $i \in [n]$, it comes directly that:
    $$ \alpha(\phi) = \max_{\mathbf{x} \in \mathtt{mods}(\phi)} ( \prod_{x_i \in \mathbf{x}} p_i \prod_{\overline{x_i} \in \mathbf{x}} (1-p_i) ) $$
    
    \item \textbf{Stratified Shortest Model} ($\mathtt{SSM}$): finds a shortest model with respect to stratified preferences defined by an ordered partition $\mathbf{P} = (\mathbf{P}_1, ... \mathbf{P}_p)$.
            \begin{itemize}
                \item Semiring $(\times_{k \in [p]} \{0,..,m_k\}, \min_{\mathtt{lex}}, +^p, \mathbf{m}, \mathbf{0}^p)$ with:
                \begin{equation*}
                        \begin{split}
                        \mathbf{a} \leq_{\mathtt{lex}} \mathbf{b} \iff & \exists k \in [p], a_k \leq b_k \\
                        \mbox{ and }  & \forall i < k, a_i=b_i \\
                        \end{split}
                    \end{equation*}
                    \begin{equation*}
                        \begin{split}
                        \forall k \in [p], & +^p(\mathbf{a}, \mathbf{b})_k = a_k + b_k \\
                        & m_k = |S_k|
                        \end{split}
                    \end{equation*}
        
                \item Attribution function: $\forall i \in [n], \alpha(x_i)= \alpha(\overline{x_i}) = \delta^k$ iff $X_i \in \mathbf{P}_k$
            \end{itemize}
        
        The associativity and commutativity of $\min_{\mathtt{lex}}$ and $+^p$ are trivial, as for the distributivity of $+^p$ over $\min_{\mathtt{lex}}$. It comes directly from the definition of the semiring that query outputs a preferred with respect to the the stratified preferences defined by the ordered partition $\mathbf{P} = (\mathbf{P}_1, ... \mathbf{P}_p)$.

    \item \textbf{Shortest Term} ($\mathtt{ST}$): Finds a shortest term in any predicate $\mathtt{pred}$ (trivial for $\mathtt{mods}$).
        \begin{itemize}
                \item Semiring $(\{0,...,n\}, \min, +, n, 0)$
                \item Attribution function: 
                \begin{equation*}
                        \forall i \in [n], \quad \alpha(x_i)= \alpha(\overline{x_i}) = 1
                \end{equation*}
            \end{itemize}
        
        The associativity and commutativity of $\min$ and $+$ are trivial, as for the distributivity of $+$ over $\min$. It comes directly from the definition of the semiring that:
        $$ \alpha(\phi) = \min_{\mathbf{t} \in \mathtt{pred}(\phi)} ( \sum_{t \in \mathbf{t}} 1) = \min_{\mathbf{t} \in \mathtt{pred}(\phi)} |\mathbf{t}|$$
\end{itemize}

Instead of counting or optimizing on models, one may be interested in changing the base of the sum $\bigoplus$ in Equation \ref{eq:amc}. Let us assume a predicate $\mathtt{pred}$ that defines a subset of terms $\mathtt{pred}(\phi, q) \subset \mathtt{terms(\mathbf{X})}$ based on an input formula $\phi$ (and potentially additional inputs $q \in \mathcal{Q}$). We define the solution to the algebraic counting query $(\mathtt{pred}, \phi, q, \mathcal{SR}, \alpha)$ as:
\begin{equation} \label{eq:apc}
    \alpha(\mathtt{pred}, \phi, q) = \bigoplus_{\mathbf{t} \in \mathtt{pred}(\phi, q)} \bigotimes_{l \in \mathbf{t}} \alpha(l)
\end{equation}

Every XAI query described earlier can be modeled as an algebraic counting query by changing the base of a query in Table \ref{tab:queries} to a specific predicate from Table \ref{tab:predicates}. For any query $\mathtt{P}$ in Table \ref{tab:queries}, we note $\mathtt{P_{pred}[L]}$ the problem of computing $\alpha(\phi)$ based on predicate $\mathtt{pred}$ for any input formula $\phi \in \mathtt{L}$, attribution function $\alpha$ and additional input $q \in \mathcal{Q}$. Additionally, we note $\mathtt{Enum_{pred}[L]}$ the problem of enumerating the terms in $\mathtt{pred}(\phi, q)$ for any $\phi \in \mathtt{L}$ and $q \in \mathcal{Q}$.

\begin{table}[h!]
\centering
\begin{tabular}{l  l}
    \toprule
    \textbf{Predicate} & \textbf{Description}\\
    \midrule
    $\mathtt{mods}(\phi)$ & Models of $\phi$ \\
    $\mathtt{ips}(\phi)$ & Implicants of $\phi$ \\
    $\mathtt{pips}(\phi)$ & Prime implicants of $\phi$ \\
    $\mathtt{axp}(\phi, \mathbf{x})$ & Abductive explanations for a model $\mathbf{x}$ of $\phi$ \\
    $\mathtt{srs}(\phi, \mathbf{x})$ & Sufficient reasons for a model $\mathbf{x}$ of $\phi$ \\
    \bottomrule
\end{tabular}
\caption{Considered predicates and their descriptions.}
\label{tab:predicates}
\end{table}

\begin{example}
    The query $\mathtt{ST_{ips}[OBDD]}$ computes a shortest implicant of an $\mathtt{OBDD}$. The query $\mathtt{CT_axp[OBDD]}$ counts the number of abductive explanations of a given model $\mathbf{x}$ of an $\mathtt{OBDD}$. The query $\mathtt{Enum_{srs}[OBDD]}$ enumerates the sufficient reasons for any model $\mathbf{x}$ of an $\mathtt{OBDD}$.
\end{example}

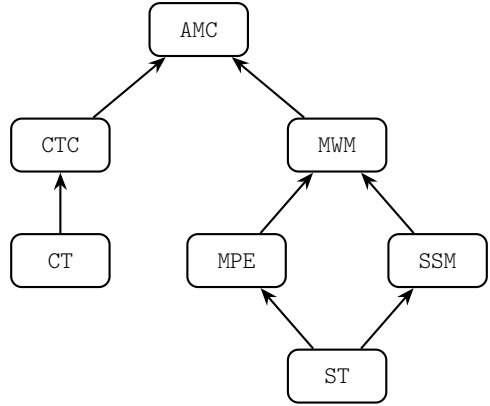
\begin{figure}[h]
    \centering
    \begin{tikzpicture}[
        node distance=0.6cm and 0.5cm,
        box/.style={draw, rounded corners, minimum width=1.3cm, minimum height=0.7cm, font=\small},
        ->, >=Stealth, thick
    ]
    
    \node[box] (AMC) {$\mathtt{AMC}$};
    
    \node[box, below left=0.8cm and 0.5cm of AMC] (CTC) {$\mathtt{CTC}$};
    \node[box, below right=0.8cm and 0.5cm of AMC] (MWM) {$\mathtt{MWM}$};
    
    \node[box, below=0.8cm of CTC] (CT) {$\mathtt{CT}$};
    \node[box, below left=0.8cm and 0.0cm of MWM] (MPE) {$\mathtt{MPE}$};
    \node[box, below right=0.8cm and 0.0cm of MWM] (SSM) {$\mathtt{SSM}$};
    
    \node[box, below right=0.8cm and 0.0cm of MPE] (ST) {$\mathtt{ST}$};

    \draw (CTC) -- (AMC);
    \draw (MWM) -- (AMC);
    \draw (CT) -- (CTC);
    \draw (MPE) -- (MWM);
    \draw (SSM) -- (MWM);
    \draw (ST) -- (MPE);
    \draw (ST) -- (SSM);
    \end{tikzpicture}
    \caption{Reductions between algebraic counting queries: an arrow $\mathtt{P_1} \to \mathtt{P_2}$ indicates that $\mathtt{P_2}$ can be reduced to $\mathtt{P_1}$.}
    \label{fig:reductions}
\end{figure}

Moreover, these queries can be organized in a lattice of reductions displayed on Figure \ref{fig:reductions}.

\begin{claim*}
    All reductions displayed on Figure \ref{fig:reductions} are true for any predicate $\mathtt{pred}$ and representation language $\mathtt{L}$.
\end{claim*}

\begin{proof}
    We prove a reduction $\mathtt{P}_1 \to \mathtt{P}_2$ by showing that any attribution function $\alpha_1$ on the semiring of $\mathtt{P}_1$ can be translated in polynomial time into an attribution function $\alpha_2$ on the semiring of $\mathtt{P}_2$ such that
    $$\alpha_1(\mathtt{pred}, \phi) = \alpha_2(\mathtt{pred}, \phi)$$
    for any predicate $\mathtt{pred}$.
    
    Since this reduction does not depend on $\mathtt{L}$, it remains true for any choice of representation language used for $\phi$.
    
    \begin{itemize}
        \item The reducibility of queries $\mathtt{CTC}$ and $\mathtt{MWM}$ to $\mathtt{AMC}$ comes directly from the selection of a specific semiring and set of attribution functions.

        \item $\mathtt{CT} \to \mathtt{CTC}$: simply sum the counts over all cardinal to get the total count (\ie $|\mathtt{mods}(\phi)| = \sum_{0 \leq k \leq n} |\mathtt{mods}_k(\phi)|$).

        \item $\mathtt{MPE} \to \mathtt{MWM}$: given an attribution function $\alpha_{\mathtt{MPE}}:\mathtt{lits}(\mathbf{X}) \to [0,1]$ for $\mathtt{MPE}$, we note $p_i = \alpha_{\mathtt{MPE}}(x_i)=1-\alpha_{\mathtt{MPE}}(\overline{x_i})$ for $i \in [n]$. Then, set the attribution function $\alpha_{\mathtt{MWM}}:\mathtt{lits}(\mathbf{X}) \to \mathbb{R} \cup \{\infty\}$ to:
        \begin{equation}
            \begin{split}
                \forall i \in [n], \quad & \alpha_{\mathtt{MWM}}(x_i)= - \log(p_i) \\
                & \alpha_{\mathtt{MWM}}(\overline{x_i})= - \log(1-p_i)
            \end{split}
        \end{equation}
        with $\log(0)=-\infty$.

    We have:
    \begin{equation*}
        \begin{split}
            \alpha_{\mathtt{MWM}}(\phi) & = \min_{\mathbf{t} \in \mathtt{pred}(\phi)} ( \sum_{l \in \mathbf{x}} \alpha(l)) \\
            & = \min_{\mathbf{t} \in \mathtt{pred}(\phi)} - ( \sum_{X_i \in \mathbf{t}} \log(p_i) + \sum_{\overline{x_i} \in \mathbf{t}} \log(1-p_i) ) \\
            & = \max_{\mathbf{t} \in \mathtt{pred}(\phi)} \exp(\sum_{X_i \in \mathbf{t}} \log(p_i) + \sum_{\overline{x_i} \in \mathbf{t}} \log(1-p_i)) \\
            & = \max_{\mathbf{t} \in \mathtt{pred}(\phi)} ( \prod_{X_i \in \mathbf{t}} p_i \prod_{\overline{x_i} \in \mathbf{t}} (1-p_i) )  \\
            & = \alpha_{\mathtt{MPE}}(\phi)
        \end{split}
    \end{equation*}

        \item $\mathtt{SSM} \to \mathtt{MWM}$: given a partition $\mathbf{P} = (\mathbf{P}_1, ... \mathbf{P}_p)$, we fix the attribution function:
            \begin{equation}
                \forall i \in [n], \quad \alpha(x_i)= \alpha(\overline{x_i}) =(n \cdot 10)^{-k} \mbox{ if } X_i \in \mathbf{P}_k
            \end{equation}

            For any terms $\mathbf{t}_1, \mathbf{t}_2$, we have:
            \begin{equation}
                \mathbf{t}_1 \prec_{\mathbf{P}} \mathbf{t}_2 \iff \sum_{l \in \mathbf{t}_1} \alpha(l) \leq \sum_{l \in \mathbf{t}_2} \alpha(l)
            \end{equation}

    

        \item $\mathtt{ST} \to \mathtt{SSM}$: fix the partition $\mathbf{P} = \mathbf{X}$.

        \item $\mathtt{ST} \to \mathtt{MPE}$: fix the attribution function $\alpha_{\mathtt{MPE}}(X_i) = \alpha_{\mathtt{MPE}}(\overline{x_i}) = \frac{1}{2}$, we have:
    \begin{equation*}
        \begin{split}
            \alpha_{\mathtt{MPE}}(\phi) & = \max_{\mathbf{t} \in \mathtt{pred}(\phi)} \prod_{l \in \mathbf{t}} \alpha_{\mathtt{MPE}}(l)\\
            & = \max_{\mathbf{t} \in \mathtt{pred}(\phi)} (\frac{1}{2})^{|\mathbf{t}|} \\
            & = - \log(\min_{\mathbf{t} \in \mathtt{pred}(\phi)} |\mathbf{t}|) \\
        \end{split}
    \end{equation*}
    \end{itemize}
\end{proof}

\section{Dual-rail encoding} \label{sec:dual-rail}
In this section, we formally define the \textbf{dual-rail encoding} and detail its key properties. We assume that $f$ is a boolean function on boolean variables $\mathbf{X}=\{X_1,...,X_n\}$. We define the set of literal variables as $\mathbf{X}^{\pm}=\{X_1^+,...,X_n^+, X_1^-,...,X_n^-\}$. Terms and instances on $\mathbf{X}^{\pm}$ are noted $\tau$ and $\chi$ respectively. A term or an instance on $\mathbf{X}^{\pm}$ is said \textbf{pure} iff it satisfies $\varkappa = \bigwedge_{1 \leq i \leq n} \neg X_i^- \lor \neg X_i^+$. For any pure term $\tau$ on $\mathbf{X}^{\pm}$, we note $\tau^{\bullet}$ the term on $\mathbf{X}$ that contains the literal $x_i$ (resp. $\overline{x_i}$) iff $\tau$ contains the literal $x_i^+$ (resp. $x_i^-$). Reciprocally, for any term $\mathbf{t}$ on $\mathbf{X}$, we note $\mathbf{t}^{\pm}$ the pure term on $\mathbf{X}^{\pm}$ that contains the literal $x_i^+$ (resp. $x_i^-$) iff $\mathbf{t}$ contains the literal $x_i$ (resp. $\overline{x_i}$). Additionally, for any term $\tau$ on $\mathbf{X}^{\pm}$, we note $\tau^0$ the instance on $\mathbf{X}^{\pm}$ where every unassigned variable in $\tau$ has been assigned to $0$. We can easily see that the function mapping each pure instance $\chi$ on $\mathbf{X}^{\pm}$ to the term $\chi^{\bullet}$ on $\mathbf{X}$ is bijective. Its inverse maps every term $\mathbf{t}$ on $\mathbf{X}$ to the pure instance $(\mathbf{t}^{\pm})^0$ on $\mathbf{X}^{\pm}$.

\begin{example} \label{ex:chi}
    Consider the term $\mathbf{t} = \overline{x_2} x_3$ on $\mathbf{X}$. Transposing to the literal domain $\mathbf{X}^{\pm}$ yields $\mathbf{t}^{\pm}=x_2^- x_3^+$. Extending to form a pure instance yields $(\mathbf{t}^{\pm})^0=\overline{x_1^+} \overline{x_1^-} \overline{x_2^+} x_2^- x_3^+ \overline{x_3^-}$.
\end{example}

We can now define the dual-rail encoding of a boolean function.
\begin{definition}
    The \textbf{dual-rail encoding} of $f$, noted $f^{\pm}$, is the boolean function on $\mathbf{X}^{\pm}$ such that:
    \begin{gather}
        f^{\pm} \models \varkappa \label{eq:purity} \\
        \chi \models f^{\pm} \iff \chi^{\bullet} \models f \label{eq:correspondance}
    \end{gather}
\end{definition}

It is well and uniquely defined since Equation \ref{eq:purity} constrains models of $f^{\pm}$ to be pure and Equation \ref{eq:correspondance} imposes a correspondence between the models of $f^{\pm}$ and the implicants of $f$.

\begin{example}[Example 1 cont'ed]
    Let $\phi^{\pm}$ be the dual-rail encoding of $\phi$. The model of $\phi^{\pm}$ corresponding to the implicant $\mathbf{t} = \overline{x_2} x_3$ is $(\mathbf{t}^{\pm})^0$ as defined in Example \ref{ex:chi}.
\end{example}





Given a model $\mathbf{x}$ of $f$, we would like to filter implicants of $f$ to only select those that are consistent with $\mathbf{x}$, \ie abductive explanations for $(\mathbf{x}, f)$.

\begin{definition}
    Given a model $\mathbf{x}$ of $f$, we note $f^{\mathbf{x}}$ the boolean function based on selector variables $\mathbf{S}=\{S_1,...,S_n\}$ such that:
    \begin{equation}
        \forall \mathbf{s} \in \{0,1\}^n, \mathbf{s} \models f^{\mathbf{x}}  \iff \mathbf{x_s} \models f
    \end{equation}
    where $\mathbf{x_s}$ is the term composed of literals $l$ such that $\mathtt{var}(l) \in \{X_i \in \mathbf{X} | s_i=1\}$.
\end{definition}

\begin{lemma}
    The function $f^{\mathbf{x}}$ is monotone.
\end{lemma}

\begin{proof}
    Let us assume two instances $\mathbf{s}, \mathbf{v}$ on $\mathbf{S}$ such that $\forall i, s_i \leq v_i$, we have:
    \begin{equation*}
        \begin{split}
            \mathbf{s} \models f^{\mathbf{x}} & \implies \mathbf{x_s} \models f \\
                                            & \implies \mathbf{x_v} \models f \mbox{ since } \mathbf{x_s} \subset \mathbf{x_v} \\\
                                            & \implies \mathbf{v} \models f^{\mathbf{x}}
        \end{split}
    \end{equation*}
    Therefore, $f^{\mathbf{x}}$ is monotone.
\end{proof}

\begin{theorem} \label{th:pips2sr}
    Prime implicants of $f^{\mathbf{x}}$ correspond to sufficient reasons for the decision $(\mathbf{x}, f)$.
\end{theorem}

\begin{proof}
    By monotonicity, a prime implicant $\mathbf{r}$ of $f^{\mathbf{x}}$ is a set of positive literals. We note $\mathbf{x_r}$ the term on $\mathbf{X}$ composed of the literals from $\mathbf{x}$ selected from $\mathbf{r}$. Then, $\mathbf{x_r}$ is an abductive explanation for $(\mathbf{x}, f)$ by definition of $f^{\mathbf{x}}$ and is subset minimal because $\mathbf{r}$ is: if a proper sub-term of $\mathbf{x_r}$ is still an abductive explanation for $(\mathbf{x}, f)$, then the corresponding proper sub-term of $\mathbf{r}$ is an implicant of $f^{\mathbf{x}}$, which is a contradiction. Therefore, prime implicants of $f^{\mathbf{x}}$ are in bijection with sufficient reasons for the decision $(\mathbf{x}, f)$.
\end{proof}

\begin{example}[Example 1 cont'ed]
    Let $\mathbf{x}=x_1 \overline{x_2} x_3$ be the same model as in Example \ref{ex:1}. Then, $\phi^{\mathbf{x}}$ has $3$ models (resp. $2$ prime implicants) corresponding to the $3$ abductive explanations (resp. $2$ sufficient reasons) for the decision $(\mathbf{x}, \phi)$. For example, $s_1s_2s_3 \in \mathtt{mods}(\phi^{\mathbf{x}})$ corresponds to $x_1 \overline{x_2} x_3 \in \mathtt{axp}(\phi, \mathbf{x})$ and $s_1 \in \mathtt{pips}(\phi^{\mathbf{x}})$ corresponds to $x_1 \in \mathtt{srs}(\phi, \mathbf{x})$.
\end{example}

\begin{theorem} \label{th:pm2x}
    Given a representation of $f^{\pm}$ in a language $\mathtt{L}$ that offers polynomial support for conditioning and a model $\mathbf{x} \in \mathtt{mods}(f)$, a representation of $f^{\mathbf{x}}$ in $\mathtt{L}$ can be obtained in polynomial time in the size of $f^{\pm}$.
\end{theorem}

\begin{proof}
    First, we note $\delta_{\mathbf{x}}$ the pure term on $\mathbf{X}^{\pm}$ that contains the literal $\overline{x_i^+}$ (resp. $\overline{x_i^-}$) iff $x_i=0$  (resp. $x_i=1$). For any pure instance $\chi$ on $\mathbf{X}^{\pm}$, $\chi$ is a model of $\delta_{\mathbf{x}}$ iff $\chi^{\bullet}$ is consistent with $\mathbf{x}$. Therefore, conditioning $f^{\pm}$ on $\delta_{\mathbf{x}}$ effectively restricts the set of considered implicants to the abductive explanations for $(\mathbf{x}, f)$. Finally, remaining variables are renamed to yield $f^{\mathbf{x}}$: any occurrence of a variable $X_i^+$ or $X_i^-$ in $\mathtt{var}(f^{\pm} |\delta_{\mathbf{x}}) = \mathbf{X}^{\pm} \setminus \mathtt{var}(\delta_{\mathbf{x}})$ is replaced with the variable $S_i$. Notice that for any $i \in [n]$ both $X_i^+$ and $X_i^-$ cannot belong together in $\mathtt{var}(f^{\pm} |\delta_{\mathbf{x}})$, therefore no two distinct variables are rename with the same variable in $\mathbf{S}$.
\end{proof}

\section{Solving hard queries using $f^{\pm}$} \label{sec:hardxai}

\paragraph{Reasoning on implicants.}
Reasoning on the implicants of a boolean function is a hard computational task in general. Deciding if a term is an implicant of a propositional formula or a boolean circuit is \textbf{coNP}-complete. Furthermore, knowing if a $\mathtt{3-DNF}$ has an implicant included in a given subset of its variables, a problem usually referred as $\mathtt{QSAT}_2$, is $\mathbf{\Sigma_2^P}$-complete and a canonical problem of the second level of the polynomial hierarchy \citep{umans_minimum_2001}. Things become much simpler for $\mathtt{d\text{-}DNNF}$ and its subsets, where both finding a model and deciding if a term is an implicant can be done in polynomial time. Since a solution can be checked in polynomial time, many decision and optimization problems on implicants fall from $\mathbf{\Sigma_2^P}$ to \textbf{NP}.

Amongst the XAI queries one might be interested to perform on the implicants of a formula, finding a shorter or optimal implicant based on explainee's preferences can be very useful. First, it provides a lower bound on the size of explanations for any model of the boolean function, which can be used as a measure of its explainability under the explainee's preferences. Secondly, optimal implicants under additive dis-utility functions can be used to emulate most probable abductive explanations of a decision taken by marginal inference on probabilistic inputs, for instance in neurosymbolic systems \citep{Manhaeve2021}.

We show in this section that most optimization queries on implicants remain \textbf{NP}-hard even when the input formula is represented as an $\mathtt{OBDD}$, one of the most tractable subset of $\mathtt{d\text{-}DNNF}$. This is for instance the case of the variant of $\mathtt{QSAT}_2$ based on $\mathtt{OBDD}$, which we call $\mathtt{QSAT_2[OBDD]}$, as shown in \citep{colnet_complexity_2022}.


\begin{lemma}[\cite{colnet_complexity_2022}] \label{lem:qsat2obdd}
    Deciding if an $\mathtt{OBDD}$ has an implicant included in a subset $\mathbf{Y} \subset \mathbf{X}$ of its variables is \textbf{NP}-hard.
\end{lemma}





We prove the hardness of finding a minimum size implicant of an $\mathtt{OBDD}$ by a reduction from $\mathtt{QSAT}_2[\mathtt{OBDD}]$ to $\mathtt{ST_{ips}[OBDD]}$, adapted from the reduction of $\mathtt{QSAT}_2$ to $\mathtt{ST_{ips}[BC]}$ found in \cite{umans_minimum_2001}.


\begin{theorem} \label{th:sipobdd}
    Computing a shortest implicant of an $\mathtt{OBDD}$ ($\mathtt{ST_{ips}[OBDD]}$) is \textbf{NP}-hard.
\end{theorem}

\begin{proof}
    Given a subset of variables $\mathbf{Y} \subset \mathbf{X}$, we assume w.l.o.g. that $\mathbf{Y}=\{X_1,...,X_m\}$ with $m \leq n$. We create additional variables $\mathbf{Z}=\{Z_1,...,Z_m\}$ and define the formula:
    \begin{equation}
        \varphi = \phi \land (\bigwedge_{i \in [m]} X_i \iff Z_i)
    \end{equation}

    If $\phi$ is provided as an $\mathtt{OBDD}$, then an $\mathtt{OBDD}$ representation of $\varphi$ can be computed in polynomial time.

    On the one hand, notice that an implicant of $\varphi$ must include all variables in $\mathbf{Y} \cup \mathbf{Z}$ to ensure that each constraint $X_i \iff Z_i$ for $i \in [m]$ is satisfied. Besides, an implicant of $\phi$ can be extracted by deleting the literals in $\mathbf{Z}$. On the other hand, an implicant on $\phi$ can be completed to form an implicant of $\varphi$. Therefore, $\phi$ has an implicant included in $\mathbf{Y}$ iff $\varphi$ has an implicant of size at most $2m$, which can be decided by a single $\mathtt{ST_{ips}[OBDD]}$ query.

    Therefore, $\mathtt{QSAT}_2[\mathtt{OBDD}]$ can be reduced in polynomial time to $\mathtt{ST_{ips}[OBDD]}$. Since $\mathtt{QSAT}_2[\mathtt{OBDD}]$ is \textbf{NP}-hard, this also proves that $\mathtt{ST_{ips}[OBDD]}$ is \textbf{NP}-hard.
\end{proof}

The combination of Theorem \ref{th:sipobdd} with the reductions displayed on Figure \ref{fig:reductions} shows that any optimization query found in Table \ref{tab:queries} made on implicants is \textbf{NP}-hard when $f$ is given as an $\mathtt{OBDD}$.

Contrastingly, assuming a representation of $f^{\pm}$ that offers polynomial support for algebraic model counting, the correspondence between the models of $f^{\pm}$ and the implicants of $f$ allows us to answer any algebraic counting query on implicants.


\begin{theorem} \label{th:aipc}
    Answering any algebraic counting query on the implicants of a boolean function $f$ given a $\mathtt{d\text{-}DNNF}$ representation $\varphi$ of $f^{\pm}$ can be done in polynomial time in $|\varphi|$.
\end{theorem}
Consequently, counting implicants or selecting a preferred implicant can be done in polynomial time.

\begin{proof}
    The proof relies on the correspondence between the implicants of $f$ and the models of $f^{\pm}$ to create a reduction from $\mathtt{AMC_{ips}[d\text{-}DNNF^{\pm}]}$ to $\mathtt{AMC_{mods}[d\text{-}DNNF]}$. Let us remind that given a semiring $\mathcal{SR}=(\mathcal{A}, \oplus, \otimes, e^{\oplus}, e^{\otimes})$ and an attribution function $\alpha: \mathtt{lits}(\mathbf{X}) \to \mathcal{A}$, the solution to the algebraic counting query $(\mathtt{imps}, f, \mathcal{SR}, \alpha)$ is:
    \begin{equation*}
        \alpha(\mathtt{imps}, f) = \bigoplus_{\mathbf{t} \in \mathtt{imps}(f)} \bigotimes_{l \in \mathbf{t}} \alpha(l).
    \end{equation*}

    Let us note $\alpha^{\pm} : \mathtt{lits}(\mathbf{X}^{\pm}) \to \mathcal{A}$ such that:
    \begin{equation*}
        \begin{split}
            \forall i \in [n], \quad & \alpha^{\pm}(x_i^+)=\alpha(x_i), \\
            & \alpha^{\pm}(x_i^-)=\alpha(\overline{x_i}), \\
            & \alpha^{\pm}(\overline{x_i^+})=\alpha^{\pm}(\overline{x_i^-})=e^{\otimes}.
        \end{split}
    \end{equation*}
    Notice that for any pure instance $\chi$ on $\mathbf{X}^{\pm}$:
    $$ \bigotimes_{l \in \chi} \alpha^{\pm}(l) = \bigotimes_{l \in \chi^{\bullet}} \alpha(l)$$
    Since the models $\chi$ of $f^{\pm}$ are pure by definition and correspond exactly to the implicants $\chi^{\bullet}$ of $f$, we have:
    \begin{equation*}
        \begin{split}
            \alpha^{\pm}(\mathtt{mods}, f^{\pm}) & = \bigoplus_{\chi \in \mathtt{mods}(f^{\pm})} \bigotimes_{l \in \chi} \alpha^{\pm}(l) \\
             & = \bigoplus_{\chi \in \mathtt{mods}(f^{\pm})} \bigotimes_{l \in \chi^{\bullet}} \alpha(l) \\
             & = \bigoplus_{\mathbf{t} \in \mathtt{imps}(f)} \bigotimes_{l \in \mathbf{t}} \alpha(l) \\
             & = \alpha(\mathtt{imps}, f)
        \end{split}
    \end{equation*}

    Therefore, if a $\mathtt{d\text{-}DNNF}$ representation $\varphi$ of $f^{\pm}$ is provided, $\alpha(\mathtt{imps}, f) = \alpha^{\pm}(\mathtt{mods}, \varphi)$ can be computed in polynomial time.
\end{proof}

\paragraph{Preferred abductive explanations.}
Instead of reasoning on all implicants of a function, computing specific abductive explanations only requires to reason about implicants that are consistent with a given input instance. The set of considered implicants being restricted, one might expect queries to be easier. Unfortunately, as for computing a shorter implicant, computing a shorter reason for a decision $(\mathbf{x}, f)$ remains hard, even for an $\mathtt{OBDD}$ representation of $f$.

\begin{lemma} \label{lemma:vc}
    Given an undirected graph $G=(V,E)$ where $V=[n]$, we can construct in polynomial time an $\mathtt{OBDD}$ $\phi$ of size $\mathcal{O}(n^2)$ such that:
    \begin{itemize}
        \item $\mathbf{1}_{[n]}$ is a model of $\phi$
        \item a subset $S \subset V$ of variables is a vertex cover of $G$ iff $\mathbf{1}_S$ is a weak abductive explanation for $(\mathbf{1}_{[n]}, \phi)$
    \end{itemize}
\end{lemma}

\begin{proof}
    The proof of this Lemma is adapted from the reduction of enumerating minimal hypergraph transversals to enumerating sufficient reasons of an $\mathtt{OBDD}$ \citep{colnet_complexity_2022}.

    Assume an undirected graph $G=(V,E)$ where $V=[n]$ and $E \subset V^2$. For each edge $(i,j) \in E$, we note $\mathbf{x}^{(i,j)}$ the instance on boolean variables $\mathbf{X}= \{X_1,...,X_n \}$ such that $x^{(i,j)}_k = 1$ if $k \not\in \{i,j\}$ and $0$ otherwise. We define the function $f^G$ on variables $\mathbf{X}$ such that $\mathtt{mods}(f^G) = \{\mathbf{x}^{(i,j)} | (i,j) \in E\}$ and we note $\delta^S = \bigvee_{i \in S} \neg X_i$ for $ S \subset [n]$. Notice that $\mathbf{x}^{(i,j)}$ is a model of $\delta^S$ iff $i \in S$ or $j \in S$. Therefore we have:
    \begin{equation} \label{eq:vcimplicates}
        \begin{split}
            f^G \models \delta^S & \iff \forall (i,j) \in E, \mathbf{x}^{(i,j)} \models \delta^S \\
            & \iff \forall (i,j) \in E, i \in S \mbox{ ou } j \in S \\
            & \iff S \mbox{ is a vertex cover of } G
        \end{split}
    \end{equation}
    In other words, the set of negative implicates of $f^G$ (\ie which only contains negative literals) correspond to vertex covers of $G$. It is known by duality that the implicants of $\neg f^G$ are exactly the negations of the implicates of $f^G$. Since the negation of a negative clause is a positive term, we have:
    \begin{equation} \label{eq:vcimplicants}
        \begin{split}
             \mathbf{1}_S \models \neg f^G & \iff f^G \models \delta^S \\
            & \iff S \mbox{ is a vertex cover of } G
        \end{split}
    \end{equation}
    Besides, notice that $\mathbf{1}_{[n]}$ is a model of $\neg f^G$ (the set $V$ to which it corresponds is clearly a vertex cover of $G$). Therefore, positive implicants of $\neg f^G$ correspond exactly to weak abductive explanations of $(\mathbf{1}_{[n]}, \neg f^G)$. Hence, a subset $S \subset V$ of variables is a vertex cover of $G$ iff $\mathbf{1}_S$ is a weak abductive explanation of $(\mathbf{1}_{[n]}, \phi)$.

    Finally, since $f^G$ has at most a quadratic number of models, it can be compiled in polynomial time into an $\mathtt{OBDD}$, which then can be negated in polynomial time to yield an $\mathtt{OBDD}$ representing $\neg f^G$.
\end{proof}

Reasoning about the vertex covers of a graph is notoriously hard. In particular, finding a minimum size vertex cover belongs to the list of canonical \textbf{NP}-hard optimization problems mentioned in \citep{karp_reducibility_1972}. Therefore, the correspondence proven in Lemma \ref{lemma:vc} directly provides a proof of \textbf{NP}-hardness for finding a minimum size sufficient reason.


\begin{corollary} \label{cor:ssrobdd}
    Computing a shortest abductive explanation for any model of an $\mathtt{OBDD}$ ($\mathtt{ST_{axp}[OBDD]}$) is \textbf{NP}-hard.
\end{corollary}

\begin{proof}
    The proof works by reduction from the minimum size vertex cover ($\mathtt{MVC}$) problem. For any undirected graph $G=(V,E)$ where $V=[n]$ and $E \subset V^2$, build an $\mathtt{OBDD}$ $\phi$ as defined in Lemma \ref{lemma:vc}. The set of vertex covers of $G$ corresponds exactly to the set of abductive explanations of $(\mathbf{1}_{[n]}, \phi)$. Moreover, the size of a vertex cover is exactly that of its corresponding sufficient reason. Therefore, finding a minimum size vertex cover can be reduced to finding a minimum size sufficient reason for $(\mathbf{1}_{[n]}, \phi)$. Since $\mathtt{MVC}$ is \textbf{NP}-hard, then $\mathtt{ST_{srs}[OBDD]}$ is also \textbf{NP}-hard.
\end{proof}

Similarly to reasoning on implicants, any query that can be reduced to $\mathtt{ST_{axp}[OBDD]}$ in Figure \ref{fig:reductions} is therefore \textbf{NP}-hard as well on abductive explanations when the input formula is given as an $\mathtt{OBDD}$.

Again, assuming a representation of $f^{\pm}$ that offers polynomial support for algebraic model counting and conditioning, the transformation from $f^{\pm}$ to $f^{\mathbf{x}}$ proven in Theorem \ref{th:pm2x} and the correspondence between the models of $f^{\mathbf{x}}$ and the abductive explanations for $(\mathbf{x}, f)$ allows us to answer any algebraic counting query on implicants.


\begin{theorem} \label{th:aaxpc}
    Answering any algebraic counting query on the abductive explanations of a model $\mathbf{x}$ of a boolean function $f$ given a $\mathtt{d\text{-}DNNF}$ representation $\varphi$ of $f^{\pm}$ can be done in polynomial time in $|\varphi|$.
\end{theorem}

\begin{proof}
    We first use Theorem \ref{th:pm2x} to compute a $\mathtt{d\text{-}DNNF}$ representation of $f^{\mathbf{x}}$ from $\varphi$. The rest of the proof is identical to that of Theorem \ref{th:aipc} by replacing $\mathtt{imps}(f)$ with $\mathtt{axp}(f, \mathbf{x})$ and $f^{\pm}$ with $f^{\mathbf{x}}$.
\end{proof}

\paragraph{Enumeration.}
Similarly to optimization queries, enumeration queries do not become easier when the set to enumerate shrinks from all prime implicants to prime implicants explanations of a specific model. Indeed, \cite{colnet_complexity_2022} showed that enumerating prime implicants of a $\mathtt{dec\text{-}DNNF}$ ($\mathtt{Enum_{pips}[dec\text{-}DNNF]}$) can be done with incremental polynomial delay. In contrast, it showed that enumerating the prime implicant explanations of an arbitrary model of an $\mathtt{OBDD}$ ($\mathtt{Enum_{srs}[OBDD]}$) is at least as hard as enumerating the minimal transversals of a hypergraph, an enumeration problem whose membership to \textbf{OutputP} is a longstanding open question in complexity theory.

We prove below that the enumeration of sufficient reasons can be done with incremental polynomial delay provided an adequate representation of the dual-rail encoding.
\begin{theorem} \label{th:enumsrspm}
    Enumerating the sufficient reasons for any model $\mathbf{x}$ of a boolean function $f$ given a $\mathtt{dec\text{-}DNNF}$ representation $\varphi$ of $f^{\pm}$ can be done with incremental polynomial delay in $|\varphi|$.
\end{theorem}

\begin{proof}
    The result comes directly from Theorems \ref{th:pips2sr} and \ref{th:pm2x} combined with the enumeration algorithm for the prime implicants of a $\mathtt{dec\text{-}DNNF}$ with incremental polynomial delay proved in \cite{colnet_complexity_2022}.
\end{proof}

\paragraph{Shapley values.}
Shapley values are a staple in the XAI literature \citep{letoffe_towards_2025}, although some of their limits have been pointed out \citep{huang_inadequacy_2023,marques-silva_explainability_2024}. Recently \citep{willot_2026} introduced a new Shapley value where the contribution of a coalition is based on abductive explanations: the score of the coalition is $1$ if selected literals form an abductive explanation for the decision, and $0$ otherwise.
\begin{definition}
    Given a boolean function $f$ and one of its model $\mathbf{x}$, the abductive Shapley value of a variable $X_i \in \mathbf{X}$ in the decision $(\mathbf{x}, f)$ is:
    \begin{equation}
        \zeta_i = \sum_{\mathbf{S} \subset [n] \setminus \{i\}} c(|\mathbf{S}|) \cdot (\nu(\mathbf{S} \cup \{i\}) - \nu(\mathbf{S}))
    \end{equation}
    where: $c(k) = \frac{k! \cdot (n - k -1)!}{n!}$ and $\nu(\mathbf{S})= \left\{\begin{array}{ll}
                                                                1 & \mbox{if } \mathbf{x_S} \models f \\
                                                                0 & \mbox{otherwise}
                                                                \end{array}\right.$
\end{definition}

A decomposition of the abductive Shapley value, similar to the one used on the standard Shapley value in \citep{arenas_tractability_2021}, shows that it can be computed if one has access to the number of abductive explanations of size $k$ that include (resp. do not include) variable $X_i$. As such counts can be computed in polynomial time given a $\mathtt{d\text{-}DNNF}$ representation of $f^{\pm}$, so can the abductive Shapley value.

\begin{theorem} \label{th:drshap}
    Given a $\mathtt{d\text{-}DNNF}$ $\varphi$ representing $f^{\pm}$ and a model $\mathbf{x}$ of $f$, one can compute the Shapley value $\zeta_i$ of a variable $X_i \in \mathbf{X}$ in the decision $(\mathbf{x}, f)$ for any $i \in [n]$ in polynomial time.
\end{theorem}

\begin{proof}
    Let us note $\eta(k)$ the number of abductive explanations of size $k$ and $\eta(k, i)$ the number of abductive explanations of size $k$ that include variable $X_i$. Notice that the number of abductive explanations of size $k$ that do not include variable $X_i$ is given by $\eta(k) - \eta(k, i)$. We then have the following decomposition:
    \begin{equation} \label{eq:shapdec}
    \begin{split}
        \psi_i & = \sum_{\mathbf{S} \subset [n] \setminus \{i\}} c(|\mathbf{S}|) \cdot (\nu(\mathbf{S} \cup \{i\}) - \nu(\mathbf{S})) \\
        & = \sum_{0 \leq k \leq n-1} c(k) \cdot \left(\sum_{\substack{\mathbf{S} \subset [n] \setminus \{i\}, \\ |\mathbf{S}|=k}} \nu(\mathbf{S} \cup \{i\} ) - \nu(\mathbf{S}) \right) \\
        & = \sum_{0 \leq k \leq n-1} c(k) \cdot (\eta(k+1, i) - \eta(k) + \eta(k, i)) \\
    \end{split}
    \end{equation}

    Hence, we can compute the abductive Shapley value if we have access to the counts $\eta(k)$ and $\eta(k, i)$.
    
    Notice that the cardinality of a model $\mathbf{s}$ of $f^{\mathbf{x}}$ is equal to the size of its corresponding abductive explanation $\mathbf{x_s}$. Therefore, computing the counts $\eta(k)$ (resp. $\eta(k, i)$) for $0 \leq k \leq n$ strictly corresponds to solving $\mathtt{CTC}$ on $f^{\mathbf{x}}$ (resp. $f^{\mathbf{x}} | \overline{s_i}$). Since a $\mathtt{d\text{-}DNNF}$ representation of $f^{\mathbf{x}}$ can be derived from $\varphi$ in polynomial time (see Theorem \ref{th:pm2x}), then counting $\eta(k)$ and $\eta(k, i)$ can be done in polynomial time.

    Finally, we can compute the abductive Shapley value in polynomial time in the size of $\phi$.
\end{proof}

Although we do not know of a formal proof of \#\textbf{P}-hardness on this specific Shapley value, the following observation regarding the complexity of counting abductive explanations ensures that the method used in the proof of Theorem \ref{th:drshap} is out of reach even for an $\mathtt{OBDD}$ representation of $f$.

\begin{theorem} \label{th:ctaxp}
    Given an $\mathtt{OBDD}$ $\phi$ and one of its model $\mathbf{x}$, counting abductive explanations of $(\mathbf{x}, \phi)$ ($\mathtt{CT_{axp}[OBDD]}$) is \textbf{\#P}-hard.
\end{theorem}

\begin{proof}
    Similarly to the proof of Corollary \ref{cor:ssrobdd}, this proof works by reduction from the problem \#$\mathtt{VC}$ of counting the vertex covers of a graph, a canonical \textbf{\#P}-hard problem \citep{valiant_complexity_1979}. For a given graph $G=(V,E)$, we build as earlier an $\mathtt{OBDD}$ such that the abductive explanations of the model $\mathbf{1}$ are in one-to-one correspondence with the set of vertex covers of $G$. Therefore, counting the vertex cover of a graph can be reduced to counting the abductive explanations for a model of an $\mathtt{OBDD}$. Therefore, we have proven that counting the abductive explanations for a model of a an $\mathtt{OBDD}$ is \textbf{\#P}-hard.
\end{proof}

\paragraph{Coverage-based explanations.}
Most explanation queries are designed under the hypothesis of variable independence. However, it is not uncommon for input variables to be tied to each other through a set of constraints that discards a subset of instances. Ignoring these input constraints can lead to superfluous and redundant explanations \citep{gorji_sufficient_2022,cooper_abductive_2023}. Therefore, new definitions of abductive explanations under constraints were introduced to cope with that phenomenon. In this section, we show how dual-rail encodings can help computing coverage-based explanations introduced in \cite{cooper_abductive_2023}.

For two boolean function $f$ and $g$, the coverage-based explanations of a model of $f$ under the input constraints $g$ is defined as follows.

\begin{definition}
    A weak abductive explanation for a model $\mathbf{x}$ of $f$ under constraints $g$ (wAXpc) is a sufficient reason for $(\mathbf{x}, f \lor \neg g)$.

    The \textbf{coverage} of a wAXpc $\mathbf{t}$ with respect to $g$ is:
    $$\mathtt{cov}_g(\mathbf{t})=\mathtt{mods}(\mathbf{t} \land f \land g)$$
    
    The \textbf{neighborhood} of $\mathbf{t}$ with respect to $g$ is the set of literals that are a direct consequence of $\mathbf{t}$ under $g$:
    $$\mathtt{nei}_g(\mathbf{t})=\{l \in \mathtt{lits}(\mathbf{X}) | \mathbf{t} \land g \models l\}$$

    A coverage-based PI-explanation for a model $\mathbf{x}$ of $f$ under constraints $g$ (CPI-Xp) is a wAXpc such that no other wAXpc has a strictly larger coverage with respect to $g$.
\end{definition}

It was shown in \cite{cooper_abductive_2023} that computing a CPI-Xp for the decision $(\mathbf{x}, f)$ under constraints $g$ can be done with a polynomial number of calls to an oracle that decides given a wAXpc if it is a CPI-Xp or not. We prove below that this test can be decided in poylnomial time if we have access to representations of $f^{\pm}$ and $g$ that share a common structure.

\begin{definition}
    A subset of variables from $\mathbf{X}^{\pm}$ is \textbf{balanced} if it contains exactly one variable between $X_i^+$ and $X_i^-$ for every $i \in [n]$. Let $v^{\pm}$ and $v$ be two vtrees on variables sets $\mathbf{X}^{\pm}$ and $\mathbf{X}$ respectively. We say that $v$ is the \textbf{base} of $v^{\pm}$ if for any balanced subset of variables $\mathbf{Z} \subset \mathbf{X}^{\pm}$, $v^{\pm}_{\mathbf{Z}} \simeq v$ where any variable $X_i^+$ or $X_i^-$ in $v^{\pm}_{\mathbf{Z}}$ is viewed as $X_i$ in $v$.
\end{definition}

By definition, if $\varphi$ is a structured $\mathtt{d\text{-}DNNF}$ on $\mathbf{X}^{\pm}$ respecting a vtree $v^{\pm}$ of base $v$, then for any instance $\mathbf{x} \in \mathcal{X}$, $\varphi | \delta_{\mathbf{x}}$ as defined in the proof of Theorem \ref{th:pm2x} is a structured $\mathtt{d\text{-}DNNF}$ respecting $v$ (where any variable $X_i^+$ or $X_i^-$ is viewed as $X_i$ in $v$).

We start with several observations before proving our main result.
\begin{lemma} \label{lem:nei}
    Given two wAXpc $\mathbf{t}$ and $\mathbf{h}$ of $(\mathbf{x}, f)$ under $g$:
    \begin{equation}
        \mathtt{cov}_g(\mathbf{t}) \subset \mathtt{cov}_g(\mathbf{h}) \iff \mathbf{h} \subset \mathtt{nei}_g(\mathbf{t})
    \end{equation}
\end{lemma}

\begin{proof}
    $(\implies)$ Let us assume that $\mathtt{cov}_g(\mathbf{t}) \subset \mathtt{cov}_g(\mathbf{h})$ and $\mathbf{h}$ contains a literal $l \not\in \mathtt{nei}_g(\mathbf{t})$. Let us note $\overline{l}$ the negation of the literal $l$. Since $\mathbf{t}$ is a wAXpc for $(\mathbf{x}, f)$ under $g$: $\mathbf{t} \models f \lor \neg g$. Thus we also have: $\mathbf{t} \cup \overline{l} \models f \lor \neg g$. We now reason on the size of $\mathtt{cov}_g(\mathbf{t} \cup \overline{l})$.

    \begin{enumerate}
        \item if $\mathtt{cov}_g(\mathbf{t} \cup \overline{l}) = \emptyset$, then $\mathbf{t} \cup \overline{l} \models \neg(f \land g)$. Hence, we have:
        \begin{equation}
        \begin{split}
            \mathbf{t} \cup \overline{l} & \models (f \lor \neg g) \land \neg(f \land g) \\
            & \models (f \lor \neg g) \land (\neg f \lor \neg g) \\
            & \models \neg g \\
        \end{split}
        \end{equation}
        Which is in contradiction with $l \not\in \mathtt{nei}_g(\mathbf{t})$.

        \item if $\mathtt{cov}_g(\mathbf{t} \cup \overline{l}) \ne \emptyset$, let us note $\mathbf{z} \in \mathtt{cov}_g(\mathbf{t} \cup \overline{l})$. We have $\mathbf{z} \in \mathtt{cov}_g(\mathbf{t})$ and $\mathbf{z} \not\in \mathtt{cov}_g(\mathbf{h})$ since $\overline{l} \in \mathbf{z}$ and $l \in \mathbf{h}$. Which is in contradiction with $\mathtt{cov}_g(\mathbf{t}) \subset \mathtt{cov}_g(\mathbf{h})$.
    \end{enumerate}

    Therefore we have:
    \begin{equation}
        \mathtt{cov}_g(\mathbf{t}) \subset \mathtt{cov}_g(\mathbf{h}) \implies \mathbf{h} \subset \mathtt{nei}_g(\mathbf{t})
    \end{equation}

    $(\Longleftarrow)$ Let us assume that $\mathbf{h} \subset \mathtt{nei}_g(\mathbf{t})$. Then, for any $\mathbf{z} \in \mathtt{cov}_g(\mathbf{t})$ we have: 
    \begin{equation}
        \begin{split}
            \mathbf{z} & \models \mathbf{t} \land f \land g \\
            & \models \mathbf{t} \land g \\
            & \models \mathtt{nei}_g(\mathbf{t}) \\
            & \models \mathbf{h} \\
        \end{split}
        \end{equation}
    Hence, in particular, $\mathbf{z} \in \mathtt{mods}(\mathbf{h} \land f \land g)=\mathtt{cov}_g(\mathbf{h})$.

    Therefore we have:
    \begin{equation}
        \mathbf{h} \subset \mathtt{nei}_g(\mathbf{t}) \implies \mathtt{cov}_g(\mathbf{t}) \subset \mathtt{cov}_g(\mathbf{h}) 
    \end{equation}
\end{proof}

\begin{corollary} \label{cor:eqcov}
    A wAXpc $\mathbf{t}$ is a CPI-Xp iff for any wAXpc $\mathbf{h} \subset \mathtt{nei}_g(\mathbf{t})$:
    $$\mathtt{cov}_g(\mathbf{t}) = \mathtt{cov}_g(\mathbf{h})$$
\end{corollary}

\begin{proof}
    We prove the contrapositive: a wAXpc $\mathbf{t}$ is not a CPI-Xp iff there is a wAXpc $\mathbf{h} \subset \mathtt{nei}_g(\mathbf{t})$ such that $\mathtt{cov}_g(\mathbf{t}) \ne \mathtt{cov}_g(\mathbf{h})$.

    First, if $\mathbf{t}$ is not a CPI-Xp, then there is another wAXpc $\mathbf{h}$ such that $\mathtt{cov}_g(\mathbf{t}) \varsubsetneq \mathtt{cov}_g(\mathbf{h})$. By Lemma \ref{lem:nei} we have that $\mathbf{h} \subset \mathtt{nei}_g(\mathbf{t})$.

    Now, if there is a wAXpc $\mathbf{h} \subset \mathtt{nei}_g(\mathbf{t})$ such that $\mathtt{cov}_g(\mathbf{t}) \ne \mathtt{cov}_g(\mathbf{h})$. By Lemma \ref{lem:nei} we have that $\mathtt{cov}_g(\mathbf{t}) \varsubsetneq \mathtt{cov}_g(\mathbf{h})$, therefore $\mathbf{t}$ is not a CPI-Xp.
\end{proof}

\begin{corollary} \label{cor:nei}
    A wAXpc $\mathbf{t}$ is a CPI-Xp iff:
    \begin{equation}
        \Psi(f, g, \mathbf{t}) \equiv \bigvee_{\substack{\mathbf{h} \subset \mathtt{nei}_g(\mathbf{t}), \\ \mathbf{h} \models f \lor \neg g}} \mathbf{h} \land f \land g \models \mathbf{t}
    \end{equation}
\end{corollary}

\begin{proof}
    First, by union over all wAXpc in $\mathtt{nei}_g(\mathbf{t})$, Corollary \ref{cor:eqcov} becomes: $\mathbf{t}$ is a CPI-Xp iff:
    $$ \mathtt{cov}_g(\mathbf{t}) = \bigcup_{\substack{\mathbf{h} \subset \mathtt{nei}_g(\mathbf{t}), \\ \mathbf{h} \models f \lor \neg g}} \mathtt{cov}_g(\mathbf{h}) $$

    Since $\mathbf{t} \subset \mathtt{nei}_g(\mathbf{t})$ and is a wAXpc, the condition can be relaxed to:
    $$ \mathtt{cov}_g(\mathbf{t}) \supset \bigcup_{\substack{\mathbf{h} \subset \mathtt{nei}_g(\mathbf{t}), \\ \mathbf{h} \models f \lor \neg g}} \mathtt{cov}_g(\mathbf{h}) $$

    Based on the definition of the coverage and switching from reasoning on models to reason about boolean formulas, the condition becomes:
    $$ \bigvee_{\substack{\mathbf{h} \subset \mathtt{nei}_g(\mathbf{t}), \\ \mathbf{h} \models f \lor \neg g}} \mathbf{h} \land f \land g \models \mathbf{t} \land f \land g $$

    Since $f \land g$ appears on both sides, the condition can be simplified to yield:
    $$ \bigvee_{\substack{\mathbf{h} \subset \mathtt{nei}_g(\mathbf{t}), \\ \mathbf{h} \models f \lor \neg g}} \mathbf{h} \land f \land g \models \mathbf{t} $$
\end{proof}

Therefore, if we can build a representation of $\Psi(f, g, \mathbf{t})$ that supports conditioning and consistency checks in polynomial time, then each literal in $\mathbf{t}$ can be tested and we can decide in polynomial time if $\mathbf{t}$ is a CPI-Xp. We show below that this can be done if we have access to structured representations of $(f \lor \neg g)^{\pm}$ and $g$ that share a common base.

\begin{theorem}
    Given a structured $\mathtt{d\text{-}DNNF}$ representation $\varphi$ of $(f \lor \neg g)^{\pm}$ respecting a vtree $v^{\pm}$ of base $v$, a structured $\mathtt{d\text{-}DNNF}$ representation $\psi$ of $g$ respecting $v$, a model $\mathbf{x} \in \mathtt{mods}(f \land g)$, and a wAXpc $\mathbf{t}$ for $(\mathbf{x}, f)$ under $g$, one can decide in polynomial time in $|\varphi| \cdot |\psi|$ if $\mathbf{t}$ is a CPI-Xp or not.
\end{theorem}

\begin{proof}
    The decision process has the following steps:
    \begin{enumerate}
        \item Check if $\psi | \mathbf{t}$ is consistent, if not then $\mathbf{t}$ has no coverage and is not a CPI-Xp.
        \item Greedily compute the set $\mathtt{nei}_g(\mathbf{t})$ using $\psi$: for each literal $l \in \mathbf{x}$, $l \in \mathtt{nei}_g(\mathbf{t})$ iff $\psi | \mathbf{t} \cup \overline{l}$ is inconsistent.
        \item Restrict the abductive explanations represented by $\varphi$ by conditioning on all literals not in $\mathtt{nei}_g(\mathbf{t})$ (as in Theorem \ref{th:pm2x} for the literals in a model):
        \begin{equation}
            \Phi(f, g, \mathbf{t}) = \varphi | \bigwedge_{l \not\in \mathtt{nei}_g(\mathbf{t})} \overline{l^{\pm}}
        \end{equation}
        \item If we replace in $\Phi(f, g, \mathbf{t})$ occurrences of the remaining variable $X_i^+$ or $X_i^-$ in $\mathtt{var}(\mathtt{nei}_g(\mathbf{t}))$ with the variable $S_i$, we retrieve a formula equivalent to $f^{\mathbf{x}} | \bigwedge_{X_i \not\in \mathtt{var}(\mathtt{nei}_g(\mathbf{t}))} \overline{s_i}$ which represents wAXpc included in $\mathtt{nei}_g(\mathbf{t})$.
        \item However, if we replace in $\Phi(f, g, \mathbf{t})$ occurrences of $X_i^+$ (resp. $X_i^-$) by $X_i$ (resp. $\neg X_i$) if $x_i=1$ (resp. $x_i=0$), then an instance $\mathbf{x} \in \mathcal{X}$ is a model of the re-written $\Phi(f, g, \mathbf{t})$ iff it is consistent with a wAXPc in $\mathtt{nei}_g(\mathbf{t})$.
        \item We can then compute $\Psi(f, g, \mathbf{t}) = \Phi(f, g, \mathbf{t}) \land \psi$ which can be done in polynomial time since both $\Phi(f, g, \mathbf{t})$ and $\psi$ respect the vtree $v$
        \item Model of $\Psi(f, g, \mathbf{t})$ exactly correspond to extensions of a wAXPc inluded in $\mathtt{nei}_g(\mathbf{t})$ that are consistent with $\psi \equiv g$, \ie
        \begin{equation}
            \Psi(f, g, \mathbf{t}) \equiv \bigvee_{\substack{\mathbf{h} \subset \mathtt{nei}_g(\mathbf{t}), \\ \mathbf{h} \models f \lor \neg g}} \mathbf{h} \land f \land g
        \end{equation}
        \item Finally test if $\Psi(f, g, \mathbf{t}) \models \mathbf{t}$: if it does, then $\mathbf{t}$ is a CPI-Xp, if not it is not.
    \end{enumerate}
\end{proof}





\section{Compiling $f^{\pm}$} \label{sec:compilation}
Every tractability theorem proved in Section \ref{sec:hardxai} relies on the availability of a compiled representation of $f^{\pm}$ in a suitable language such as $\mathtt{d\text{-}DNNF}$, $\mathtt{dec\text{-}DNNF}$ or $\mathtt{sd\text{-}DNNF}$. Most knowledge compilers work on input formulas written as $\mathtt{CNF}$. Since the dual-rail encoding of a $\mathtt{CNF}$ can be easily represented as a $\mathtt{CNF}$ formula itself \citep{previti_prime_nodate}, existing knowledge compilers can be used as they stand.

\begin{lemma}
    Given an input $\mathtt{CNF}$, a $\mathtt{CNF}$ representation of its dual-rail encoding can be computed in linear time.
\end{lemma}

\begin{proof}
    Given a $\mathtt{CNF}$ $\Gamma = \bigwedge_{C \in \Gamma} \bigvee_{l \in C} l$, we note:
        \begin{equation}
            \Sigma = \left( \bigwedge_{C \in \Gamma} \bigvee_{l \in C} l^{\pm} \right) \land \left( \bigwedge_{i \in [n]} \neg X_i^+ \lor \neg X_i^- \right) 
        \end{equation}

    First, $\Sigma$ is trivially in $\mathtt{CNF}$. It has a size linear in that of $\Gamma$ since every clause keeps the same size and the added constraint $\varkappa$ has a size linear in $n$. It can be computed in linear time by renaming variables in the input $\mathtt{CNF}$ $\Gamma$ and adding a linear constraint $\varkappa$.
    
    All models of $\Sigma$ are pure since $\Sigma \models \varkappa$ is straightforward.

    Finally, let us assume a model $\chi \in \mathtt{mods}(\Sigma)$ and prove that $\chi^{\bullet}$ is an implicant of $\Gamma$. For any clause $C \in \Gamma$, we have: $\chi \models \bigvee_{l \in C} l^{\pm}$ which means that $\chi$ and $\bigvee_{l \in C} l^{\pm}$ have at least a literal $\lambda = l^{\pm}$ in common. Hence, $\chi^{\bullet}$ and $C$ have at least the literal $l$ in common, which means that $\chi^{\bullet} \models C$. Since this is true for all $C \in \Gamma$, we have:
    \begin{equation}
        \chi^{\bullet} \models \bigwedge_{C \in \Gamma} C \equiv \Gamma
    \end{equation}

    Therefore, $\Sigma$ is a linear size $\mathtt{CNF}$ representation of $\Gamma^{\pm}$ and can be computed in linear time.
\end{proof}

More importantly, the $\mathtt{CNF}$ representation of the dual-rail encoding preserves most of the structure of the original $\mathtt{CNF}$.
\begin{theorem}
    Let $\Gamma$ be a $\mathtt{CNF}$ on $\mathbf{X}$ and $\Sigma$ be its dual-rail encoding in $\mathtt{CNF}$. Let $w$ be a width measure between incidence/primal tree-width/path-width, we have:
    $$ w(\Sigma) \leq 2 \cdot w(\Gamma) +1 $$
\end{theorem}

\begin{proof}
    Any tree-decomposition of the primal (res. incidence) hypergraph of $\Gamma$ can be turned into a tree-decomposition of the primal (res. incidence) hypergraph of $\Sigma$ by replacing every variable $X_i$ by the pair of variables $(X_i^+, X_i^-)$, which only doubles the size of the decomposition.
\end{proof}

Based on this observation, every theorem in Section \ref{sec:hardxai} that offers polynomial support for XAI queries based on a compiled representation of $f^{\pm}$ can be turned into a fixed-parameter tractable algorithm for $\mathtt{CNF}$ based on a standard width parameter.

\begin{figure}[t]
    \centering
    \begin{subfigure}[b]{0.32\textwidth}
        \centering
        \includegraphics[width=\textwidth]{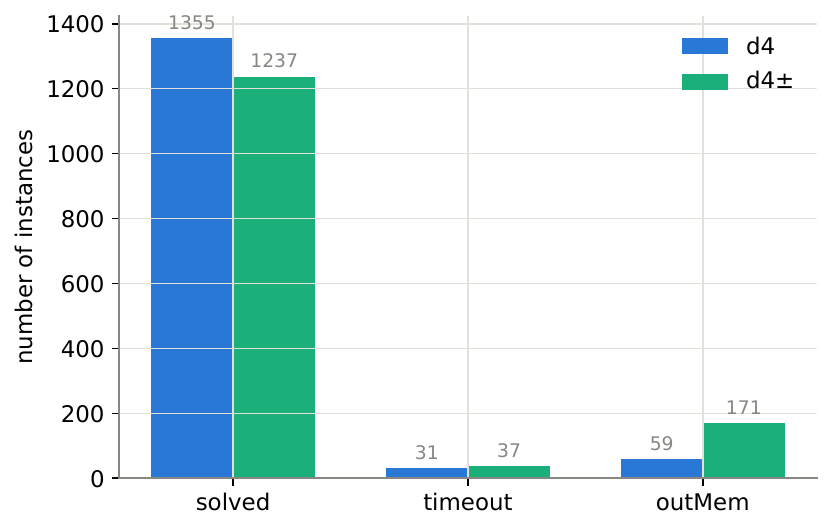}
        \caption{Number of solved instances, timeouts, and memory-outs.}
        \label{fig:outcome}
    \end{subfigure}
    \hfill
    \begin{subfigure}[b]{0.32\textwidth}
        \centering
        \includegraphics[width=\textwidth]{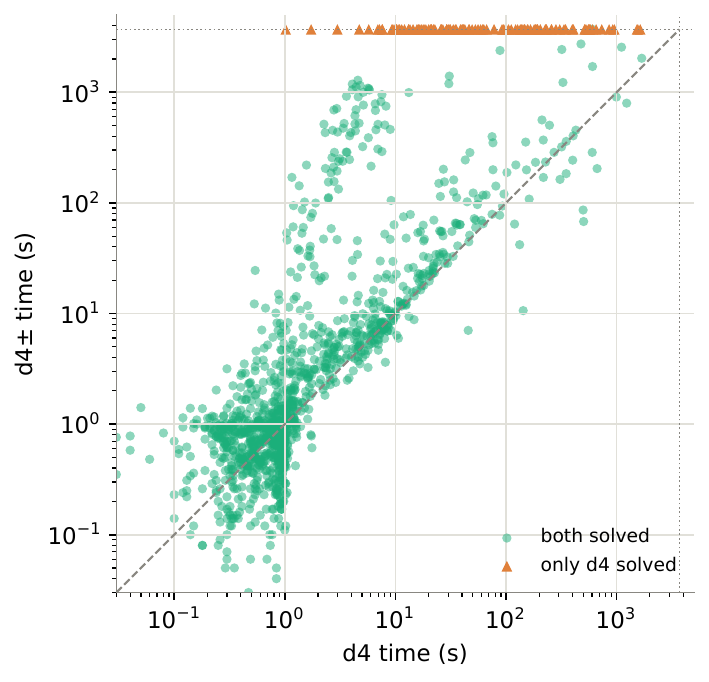}
        \caption{Compilation time (s), per instance.}
        \label{fig:time}
    \end{subfigure}
    \hfill
    \begin{subfigure}[b]{0.32\textwidth}
        \centering
        \includegraphics[width=\textwidth]{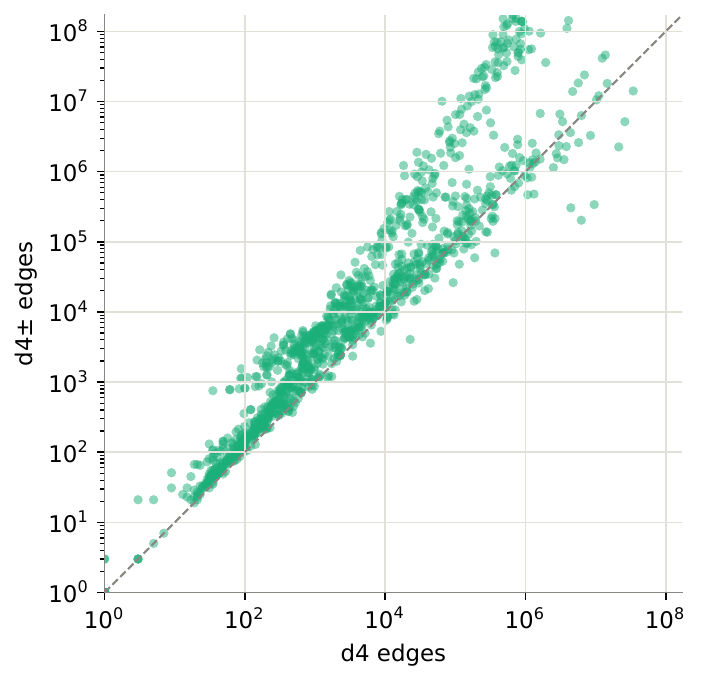}
        \caption{Size of the $\mathtt{d\text{-}DNNF}$ (number of edges).}
        \label{fig:edges}
    \end{subfigure}
    \caption{Comparison of \dfour{} (compiling the CNF) and \dfourpm{} (compiling its dual-rail encoding). In the scatter plots~(b) and~(c), each dot is an instance, the $x$-axis reports \dfour{} and the $y$-axis \dfourpm{}; instances solved only by \dfour{} are pinpointed on the plots.}
    \label{fig:xp}
\end{figure}

\section{Experiments}  \label{sec:xp}

In order to evaluate the impact of the dual-rail encoding on the compilation time and the size of a representation that allows explainability queries to be answered in polynomial time, we rely on benchmarks suited for knowledge compilation.
More precisely, our evaluation uses the benchmarks from~\cite{DBLP:conf/jelia/LagniezL25}.\footnote{\url{https://zenodo.org/records/15837216}}
All experiments were conducted on a cluster equipped with Intel\textsuperscript{\textregistered} Xeon\textsuperscript{\textregistered} {E5}-2643 v4 CPUs (3.30\,GHz) running Rocky Linux 9.5 (kernel 5.14).
For each compilation, we set a time limit of 3600 seconds and a memory limit of 32\,GB per instance.
We implement the dual-rail mechanism on top of the \dfour{} compiler~\cite{DBLP:conf/ijcai/LagniezM17}.\footnote{\url{https://github.com/jm62300/d4}}
We compare two approaches: the baseline, denoted \dfour{}, which compiles the input CNF formula directly, and our method, dual-rail \dfour{} (denoted \dfourpm{}), which first transforms the CNF using the dual-rail encoding and then compiles the resulting formula $f^{\pm}$ with \dfour{}.
To ease reproducibility, the supplementary material provides the statically compiled binary used to run the experiments, together with the raw benchmark logs and the scripts that regenerate the figures reported below.


Figure~\ref{fig:xp} summarizes our results.
Figure~\ref{fig:outcome} reports, for each approach, the number of solved instances, timeouts, and memory-outs.
\dfour{} solves $1355$ instances, with $31$ timeouts and $59$ memory-outs, whereas \dfourpm{} solves $1237$ instances, with $37$ timeouts and $171$ memory-outs.
As expected, \dfourpm{} solves fewer instances than \dfour{}: it produces a richer representation that supports queries which cannot be answered from the plain $\mathtt{CNF}$ compilation.
The main bottleneck is memory consumption, as witnessed by the substantially larger number of memory-outs.

Figure~\ref{fig:time} is a scatter plot comparing the compilation time of the two approaches, each dot being an instance, with \dfour{} on the $x$-axis and \dfourpm{} on the $y$-axis; instances solved only by \dfour{} are pinpointed.
Although \dfourpm{} is slower, handling the additional queries typically increases the compilation time by about one order of magnitude.
This overhead is paid offline, once, and is largely compensated by the fact that the resulting representation answers in polynomial time queries that are \textbf{NP}-hard without the dual-rail encoding.

Figure~\ref{fig:edges} uses the same axis convention to compare the size of the $\mathtt{d\text{-}DNNF}$ produced by the two approaches, measured as the number of edges.
The dual-rail encoding generally increases the size of the $\mathtt{d\text{-}DNNF}$, but the growth remains manageable: queries are then answered in polynomial time on this structure, whereas they cannot be handled in general on the compilation of the plain $\mathtt{CNF}$.

\section{Conclusion} \label{sec:conclusion}
The contributions of the paper are two-folds. First, we prove the hardness of a several XAI queries even for an $\mathtt{OBDD}$ representation of the classifier $f$, one of the most tractable language in the knowledge compilation map. Then, we demonstrate the potential of the dual-rail encoding for answering these queries: we proved for instance that a $\mathtt{d\text{-}DNNF}$ representation of the dual-rail can answer any algebraic counting query on implicants or abductive explanations for positive decisions. Included as special cases are finding a preferred abductive explanation with respect to quantified or stratified preferences or counting abductive explanations to compute a new Shapley value based on abduction. Moreover, we showed that a $\mathtt{dec\text{-}DNNF}$ representation of the dual-rail provides an enumeration algorithm for the sufficient reasons of positive decisions with incremental polynomial delay. Finally, sharing structure between the representations of the dual-rail and domain constraints allows us to compute coverage-based explanations.

Future research directions included a thorough experimental analysis of the runtimes of computing explanations based on a compiled dual-rail representation and a comparison with solver or oracle based solutions. In particular, the impact of compilation on enumeration delays is a promising use-case. Exploring compilation of the dual-rail for non-clausal formulas is also an interesting avenue of research.


\bibliography{biblio}

\end{document}